\documentclass[letterpaper]{article}
\usepackage[margin=1in]{geometry}
\usepackage{amsfonts}
\newcommand{\mb}{\mathbf}
\newcommand{\bs}{\boldsymbol}
\usepackage{graphicx}
\usepackage{amsmath}
\usepackage{amsthm}
\usepackage{amsfonts}
\newtheorem{theorem}{Theorem}
\newtheorem{lemma}{Lemma}

\newtheorem{definition}{Definition}
\newtheorem{condition}{Condition}

\newcommand{\mc}{\mathcal}
\usepackage{url}
\newcommand{\mr}{\mathrm}
\newcommand{\E}{\mathbb{E}}
\newcommand{\tr}{\mathrm{tr}}
\usepackage{balance}

\usepackage{graphicx} 
\usepackage{subfigure}
\usepackage{amsmath}
\usepackage{amsthm}
\usepackage{amsfonts}
\usepackage{authblk}

\usepackage{color}
\usepackage{algorithm}
\usepackage{algorithmic}

\usepackage{wrapfig}
\newcommand{\tb}{\textbf}
\newcommand{\ti}{\textit}
\usepackage{enumitem}
\usepackage{multirow}
\usepackage{amsmath}
\usepackage{amsthm}
\usepackage{amsfonts}

\usepackage{caption}
\newcommand{\spann}{\mathrm{span}}
\newcommand{\cond}{\mathrm{cond}}
\newcommand{\nn}{\nonumber}
\title{Orthogonality-Promoting Distance Metric Learning: Convex Relaxation and Theoretical Analysis}

\author[$\dag$*]{Pengtao Xie}
\author[*]{Wei Wu}
\author[$\S$]{Yichen Zhu}
\author[$\dag$]{Eric P. Xing}
\affil[$\dag$]{Petuum Inc.}
\affil[*]{School of Computer Science, Carnegie Mellon University}
\affil[$\S$]{School of Mathematical Sciences, Peking University}
\affil[ ]{\{pengtao.xie, eric.xing\}@petuum.com, weiwu2@cs.cmu.edu, acezyc@pku.edu.cn}

\date{}

\begin{document}

%

%


\maketitle

\begin{abstract}
Distance metric learning (DML), which learns a distance metric from labeled ``similar'' and ``dissimilar'' data pairs, is widely utilized. Recently, several works investigate orthogonality-promoting regularization (OPR), which encourages the projection vectors in DML to be close to being orthogonal, to achieve three effects: (1) high balancedness -- achieving comparable performance on both frequent and infrequent classes; (2) high compactness -- using a small number of projection vectors to achieve a ``good'' metric; (3) good generalizability -- alleviating overfitting to training data. While showing promising results, these approaches suffer three problems. First, they involve solving non-convex optimization problems where achieving the global optimal is NP-hard. Second, it lacks a theoretical understanding why OPR can lead to balancedness. Third, the current generalization error analysis of OPR is not directly on the regularizer. In this paper, we address these three issues by (1) seeking convex relaxations of the original nonconvex problems so that the global optimal is guaranteed to be achievable; (2) providing a formal analysis on OPR's capability of promoting balancedness; (3)  providing a theoretical analysis that directly reveals the relationship between OPR and generalization performance. Experiments on various datasets demonstrate that our convex methods are more effective in promoting balancedness, compactness, and generalization, and are computationally more efficient, compared with the nonconvex methods.
\end{abstract}

\section{Introduction}

Given data pairs labeled as either ``similar" or ``dissimilar", distance metric learning~\cite{xing2002distance,weinberger2005distance,davis2007information} learns a distance measure in such a way that
similar examples are placed close to each other while dissimilar ones are separated apart. The learned distance metrics are important to many downstream tasks, such as retrieval~\cite{chen2017diversity}, classification~\cite{weinberger2005distance} and clustering~\cite{xing2002distance}.
One commonly used distance metric between two examples $\mb{x},\mb{y}\in\mathbb{R}^D$ is: $\|\mb{A}\mb{x}-\mb{A}\mb{y}\|_{2}$~\cite{weinberger2005distance,xie2015learning,chen2017diversity}, which is parameterized by $R$ projection vectors (in $\mb{A}\in\mathbb{R}^{R\times D})$.

Many works~\cite{wang2012semi,xie2015learning,wang2015deep,raziperchikolaeilearning2016,chen2017diversity} have proposed orthogonality-promoting DML to learn distance metrics that are (1) \ti{balanced}: performing equally well on data instances belonging to frequent and infrequent classes; (2) \ti{compact}: using a small number of projection vectors to achieve a ``good" metric, (i.e., capturing well the relative distances of the data pairs); (3) \ti{generalizable}: reducing the overfitting to training data. Regarding balancedness, under many circumstances, the frequency of classes, defined as the number of examples belonging to each class, can be highly imbalanced. 
Classic DML methods are sensitive to the skewness of the frequency of the classes: they perform favorably on frequent classes whereas less well on infrequent classes --- a phenomenon also confirmed in our experiments in Section \ref{sec:exp}. However, infrequent classes are of crucial importance in many applications, and should not be ignored. 
For example, in a clinical setting, many diseases occur infrequently, but are life-threatening. 
Regarding compactness, the number of the projection vectors $R$ entails a tradeoff between performance and computational complexity~\cite{ge2014graph,xie2015learning,raziperchikolaeilearning2016}. On one hand, more projection vectors bring in more expressiveness in measuring distance. 
On the other hand, a larger $R$ incurs a higher computational overhead since the number of weight parameters in $\mb{A}$ grows linearly with $R$. It is therefore desirable to keep $R$ small without hurting much ML performance. Regarding generalization performance, in the case where the sample size is small but the size of $\mb{A}$ is large, overfitting can easily happen.

To address these three issues, many studies~\cite{xing2002distance,weinberger2005distance,davis2007information,guillaumin2009you,ying2012distance,kostinger2012large,zadehgeometric} propose to regularize the projection vectors to be close to being orthogonal. For balancedness, they argue that, without orthogonality-promoting regularization (OPR), the majority of projection vectors learn latent features for frequent classes since these classes have dominant signals in the dataset; through OPR, the projection vectors uniformly ``spread out'', giving both infrequent and frequent classes a fair treatment and thus leading to a more balanced distance metric (see \cite{xienear} for details). For compactness, they claim that: ``diversified'' projection vectors bear less redundancy and are mutually complementary; as a result, a small number of such vectors are sufficient to achieve a ``good'' distance metric. For generalization performance, they posit that OPR imposes a structured constraint on the function class of DML, hence reduces model complexity.

While these orthogonality-promoting DML methods have shown promising results, they have three problems. First, they involve solving non-convex optimization problems where the global solution is extremely difficult, if not impossible, to obtain. 
Second, no formal analysis is conducted regarding why OPR can promote balancedness. Third, while the generalization error (GE) analysis of OPR has been studied in \cite{xienear}, it is incomplete. In this analysis, they first show that the upper bound of GE is a function of cosine similarity (CS), then show that CS and the regularizer are somewhat aligned in shape. They did not establish a direct relationship between the GE bound and the regularizer. 

In this paper, we aim at addressing these problems by making the following contributions:
\begin{itemize}[leftmargin=*]
    \item We relax the nonconvex, orthogonality-promoting DML problems into convex problems and develop efficient proximal gradient descent algorithms. The algorithms only run once with a single initialization, and hence are much more efficient than existing non-convex methods.
    \item We perform theoretical analysis which formally reveals the relationship between OPR and balancedness: stronger OPR leads to more balancedness.
    \item We perform generalization error (GE) analysis which shows that reducing the convex orthogonality-promoting regularizers can reduce the upper bound of GE.
    \item We apply the learned distance metrics for information retrieval to healthcare, texts, images, and sensory data. Compared with non-convex baseline methods, our approaches achieve higher computational efficiency and are more capable of improving balancedness, compactness and generalizability.
\end{itemize}

\section{Related Works}
\subsection{Distance Metric Learning}
Many studies \cite{xing2002distance,weinberger2005distance,davis2007information,guillaumin2009you,ying2012distance,kostinger2012large,zadehgeometric} have investigated DML. Please refer to~\cite{kulis2013metric,wang2015survey} for a detailed review. Xing et al.~\cite{xing2002distance} learn a Mahalanobis distance by minimizing the sum of distances of all similar data pairs subject to the constraint that the sum of all dissimilar pairs is no less than 1. Weinberger et al.~\cite{weinberger2005distance} propose large margin metric learning, which is applied for k-nearest neighbor classification. For each data example $\mb{x}_i$, they first obtain $l$ nearest neighbors based on Euclidean distance. Among the $l$ neighbors, some (denoted by $\mb{S}=\{\mb{x}_j\}$) have the same class label with $\mb{x}_i$ and others (denoted by $D=\{\mb{x}_k\}$) do not. Then a projection matrix $\mb{L}$ is learned such that $\|\mb{L}(\mb{x}_i-\mb{x}_k)\|_2^2-\|\mb{L}(\mb{x}_i-\mb{x}_j)\|_2^2\geq 1$ where $\mb{x}_j\in \mb{S}$ and $\mb{x}_k\in \mb{D}$. Davis et al.~\cite{davis2007information} learn a Mahalanobis distance such that the distance between similar pairs is no more than a threshold $s$ and the distance between dissimilar pairs is no greater than a threshold $t$. Guillaumin et al.~\cite{guillaumin2009you} define a probability of the similarity/dissimilarity label conditioned on the Mahalanobis distance: $p(y_{ij}|(\mb{x}_i-\mb{x}_j)^\top\mb{M}(\mb{x}_i-\mb{x}_j))=1/(1+\exp((\mb{x}_i-\mb{x}_j)^\top\mb{M}(\mb{x}_i-\mb{x}_i)))$, where the binary variable $y_{ij}$ equals to 1 if $\mb{x}_i$ and $\mb{x}_j$ have the same class label. $\mb{M}$ is learned by maximizing the conditional likelihood of the training data. Kostinger et al.~\cite{kostinger2012large} learn a Mahalanobis distance metric from equivalence constraints based on likelihood ratio test. The Mahalanobis matrix is computed in one shot, without going through an iterative optimization procedure. Ying and Li~\cite{ying2012distance} formulate DML as an eigenvalue optimization problem. Zadeh et al.~\cite{zadehgeometric} propose a geometric mean metric learning approach, based on the Riemannian geometry of positive definite
matrices. Similar to \cite{kostinger2012large}, the Mahalanobis matrix has a closed form solution without iterative optimization.

To avoid overfitting in DML, various regularization approaches have been explored. Davis et al.~\cite{davis2007information} regularize the Mahalanobis matrix to be close to another matrix that encodes prior information, where the closeness is measured using log-determinant divergence. Qi et al.~\cite{qi2009efficient} use $\ell_{1}$ regularization to learn sparse distance metrics for high-dimensional, small-sample problems. Ying et al.~\cite{ying2009sparse} use $\ell_{2,1}$ norm to simultaneously encourage low-rankness and sparsity. Trace norm is leveraged to encourage low-rankness in~\cite{niu2012information,liu2015low}. Qian et al.~\cite{qian2014distance} apply dropout to DML. Many works~\cite{wang2012semi,ge2014graph,xie2015learning,yao2016diversity,raziperchikolaeilearning2016,chen2017diversity} study diversity-promoting regularization in DML or hashing. They define regularizers based on squared Frobenius norm~\cite{wang2012semi,fu2014nokmeans,ge2014graph,chen2017diversity} or angles~\cite{xie2015learning,yao2016diversity} to encourage the projection vectors to approach orthogonal. Several works~\cite{liu2008output,weiss2009spectral,gong2013iterative,ji2014batch,wang2015deep} impose strict orthogonal constraint on the projection vectors. As observed in previous works~\cite{wang2012semi,fu2014nokmeans} and our experiments, strict orthogonality hurts performance. Isotropic hashing~\cite{kong2012isotropic,ge2014optimized} encourages the variances of different projected dimensions to be equal to achieve balance. Carreira-Perpin{\'a}n and Raziperchikolaei~\cite{carreira2016ensemble} propose a diversity hashing method which first trains hash functions independently and then introduces diversity among them based on classifier ensembles.

\subsection{Orthogonality-Promoting Regularization}

Orthogonality-promoting regularization has been studied in other problems as well, including ensemble learning, latent variable modeling, classification and multitask learning. In ensemble learning, many studies \cite{kuncheva2003measures,banfield2005ensemble,partalas2008focused,yu2011diversity} promote orthogonality among the coefficient vectors of base classifiers or regressors, with the aim to improve generalization performance and reduce computational complexity. Recently, several works \cite{zou2012priors,bao2013incoherent,cogswell2015reducing,xie2015diversifying} study orthogonality-promoting regularization of latent variable models (LVMs), which encourages the components in LVMs to be mutually orthogonal, for the sake of capturing infrequent patterns and reducing the number of components without sacrificing modeling power. In these works, various orthogonality-promoting regularizers have been proposed, based on determinantal point process~\cite{kulesza2012determinantal,zou2012priors} and cosine similarity~\cite{yu2011diversity,bao2013incoherent,xie2015diversifying}. In multi-way classification, Malkin and Bilmes~\cite{malkin2008ratio} propose to use the determinant of a covariance matrix to encourage orthogonality among classifiers. Jalali et al.~\cite{jalali2015variational} propose a class of \ti{variational Gram functions} (VGFs) to promote pairwise orthogonality among vectors. While these VGFs are convex, they can only be applied to non-convex DML formulations. As a result, the overall regularized DML is non-convex and is not amenable for convex relaxation.

In the sequel, we review two families of orthogonality-promoting regularizers.
\paragraph{Determinantal Point Process}
\cite{zou2012priors} employed the determinantal point process (DPP) \cite{kulesza2012determinantal} as a prior to induce orthogonality in latent variable models. DPP is defined over $K$ vectors: $p(\{\mb{a}_i\}_{i=1}^{K})\propto \textrm{det}(\mb{L})$, where $\mb{L}$ is a $K\times K$ kernel matrix with $L_{ij}=k(\mb{a}_i,\mb{a}_j)$ and $k(\cdot,\cdot)$ as a kernel function. $\textrm{det}(\cdot)$ denotes the determinant of a matrix. A configuration of $\{\mb{a}_i\}_{i=1}^{K}$ with larger probability is deemed to be more orthogonal. The underlying intuition is that: $\textrm{det}(\mb{L})$ represents the volume of the parallelepiped formed by vectors in the kernel-induced feature space. If these vectors are closer to being orthogonal, the volume is larger, which results in a larger $p(\{\mb{a}_i\}_{i=1}^{K})$. The shortcoming of DPP is that it is sensitive to vector scaling. Enlarging the magnitudes of vectors results in larger volume, but does not essentially affects the orthogonality of vectors.

\paragraph{Pairwise Cosine Similarity}
Several works define orthogonality-promoting regularizers based on the pairwise cosine similarity among component vectors: if the cosine similarity scores are close to zero, then the components are closer to being orthogonal. Given $K$ component vectors, the cosine similarity $s_{ij}$ between each pair of components $\mb{a}_i$ and $\mb{a}_j$ is computed: $s_{ij}=\mb{a}_i\cdot\mb{a}_j/(\|\mb{a}_i\|_2\|\mb{a}_j\|_2)$. Then these scores are aggregated as a single score. In \cite{yu2011diversity}, these scores are aggregated as $\sum_{1\leq i< j\leq K}(1-s_{ij})$. In \cite{bao2013incoherent}, the aggregation is performed as $-\log(\frac{1}{K(K-1)}\sum_{1\leq i< j\leq K}\beta |s_{ij}|)^{\frac{1}{\beta}}$ where $\beta>0$. In \cite{xie2015diversifying}, the aggregated score is defined as mean of $\arccos(|s_{ij}|)$ minus the variance of $\arccos(|s_{ij}|)$.

\section{Preliminaries}
We review a DML method~\cite{xienear} that uses BMD~\cite{kulis2009low} to promote orthogonality. 

\paragraph{Distance Metric Learning}
Given data pairs labeled either as ``similar'' $\mathcal{S}=\{(\mb{x}_{i},\mb{y}_{i})\}_{i=1}^{|\mathcal{S}|}$ or ``dissimilar'' $\mathcal{D}=\{(\mb{x}_{i},\mb{y}_{i})\}_{i=1}^{|\mathcal{D}|}$, DML \cite{xing2002distance,weinberger2005distance,davis2007information} aims to learn a distance metric under which similar examples are close to each other and dissimilar ones are separated far apart. There are many ways to define a distance metric. Here, we present two popular choices. One is based on linear projection~\cite{weinberger2005distance,xie2015learning,chen2017diversity}. Given two examples $\mb{x},\mb{y}\in\mathbb{R}^D$, a linear projection matrix $\mb{A}\in\mathbb{R}^{R\times D}$ can be utilized to map them into a $R$-dimensional latent space. The distance metric is then defined as their squared Euclidean distance in the latent space: $\|\mb{Ax}-\mb{Ay}\|_2^2$. $\mb{A}$ can be learned by minimizing~\cite{xing2002distance}: $\frac{1}{|\mathcal{S}|}\sum_{(\mb{x},\mb{y})\in \mathcal{S}}\|\mb{Ax}-\mb{Ay}\|_2^2+\frac{1}{|\mathcal{D}|} \sum_{(\mb{x},\mb{y})\in \mathcal{D}}\textrm{max}(0,\tau-\|\mb{Ax}-\mb{Ay}\|_2^2)$,
which aims at making the distances between similar examples as small as possible while separating dissimilar examples with a margin $\tau$ using a hinge loss. We call this formulation as \ti{projection matrix-based DML} (PDML). PDML is a non-convex problem where the global optimal is difficult to achieve. Moreover, one needs to manually tune the number of projection vectors, typically via cross-validation, which incurs substantial computational overhead.

The other popular choice of distance metric is $(\mb{x}-\mb{y})^\top\mb{M}(\mb{x}-\mb{y})$, which is cast from $\|\mb{Ax}-\mb{Ay}\|_2^2$ by replacing $\mb{A}^\top\mb{A}$ with a positive semidefinite (PSD) matrix $\mb{M}$. This is known as the Mahalanobis distance~\cite{xing2002distance}. Correspondingly, the PDML formulation can be transformed into a \ti{Mahalanobis distance-based DML} (MDML) problem: $\textrm{min}_{\mb{M}\succeq 0}\; \frac{1}{|\mathcal{S}|}\sum_{(\mb{x},\mb{y})\in \mathcal{S}}(\mb{x}-\mb{y})^\top\mb{M}(\mb{x}-\mb{y})+\frac{1}{|\mathcal{D}|} \sum_{(\mb{x},\mb{y})\in \mathcal{D}}\textrm{max}(0, \tau-(\mb{x}-\mb{y})^\top\mb{M}(\mb{x}-\mb{y}))$,
which is a convex problem where the global solution is guaranteed to be achievable. It also avoids tuning the number of projection vectors. However, the drawback of this approach is that, in order to satisfy the PSD constraint, one needs to perform eigen-decomposition of $\mb{M}$ in each iteration, which incurs $O(D^3)$ complexity.

\paragraph{Orthogonality-Promoting Regularization}
Among the various orthogonality-promoting regularizers, we choose the BMD~\cite{kulis2009low} regularizer~\cite{xienear} in this study since it is amenable for convex relaxation and facilitates theoretical analysis.

To encourage orthogonality between two vectors $\mb{a}_i$ and $\mb{a}_j$, one can make their inner product $\mb{a}_i^\top \mb{a}_j$ close to zero and their $\ell_2$ norm $\|\mb{a}_i\|_2$, $\|\mb{a}_j\|_2$ close to one. For a set of vectors $\{\mb{a}_i\}_{i=1}^{R}$, their near-orthogonality can be achieved by computing the Gram matrix $\mb{G}$ where $G_{ij}=\mb{a}_i^\top \mb{a}_j$, then encouraging $\mb{G}$ to be close to an identity matrix. Off the diagonal of $\mb{G}$ and $\mb{I}$ are $\mb{a}_i^\top \mb{a}_j$ and zero, respectively. On the diagonal of $\mb{G}$ and $\mb{I}$ are $\|\mb{a}_i\|_2^2$ and one, respectively. Making $\mb{G}$ close to $\mb{I}$ effectively encourages $\mb{a}_i^\top \mb{a}_j$ to be close to zero and $\|\mb{a}_i\|_2$ close to one, which therefore encourages $\mb{a}_i$ and $\mb{a}_j$ to be close to orthogonal.

BMDs can be used to measure the ``closeness'' between two matrices. Let $\mb{S}^n$ denote real symmetric $n\times n$ matrices. Given a strictly convex, differentiable function $\phi: \mb{S}^n\to\mathbb{R}$, a BMD is defined as $\Gamma_{\phi}(\mb{X},\mb{Y})=\phi(\mb{X})-\phi(\mb{Y})-\text{tr}((\bigtriangledown\phi(\mb{Y}))^\top(\mb{X}-\mb{Y}))$, where $\text{tr}(\mb{A})$ denotes the trace of matrix $\mb{A}$. Different choices of $\phi(\mb{X})$ lead to different divergences. When $\phi(\mb{X})=\|\mb{X}\|_F^2$, the BMD is specialized to the \ti{squared Frobenius norm} (SFN) $\|\mb{X}-\mb{Y}\|_F^2$. If $\phi(\mb{X})=\text{tr}(\mb{X}\log \mb{X}-\mb{X})$, where $\log \mb{X}$ denotes the matrix logarithm of $\mb{X}$, the divergence becomes $\Gamma_{vnd}(\mb{X},\mb{Y})=\text{tr}(\mb{X}\log \mb{X}-\mb{X}\log \mb{Y}-\mb{X}+\mb{Y})$, which is referred to as \ti{von Neumann divergence} (VND)~\cite{tsuda2005matrix}. If $\phi(\mb{X})=-\log\det \mb{X}$ where $\det(\mb{X})$ denotes the determinant of $\mb{X}$, we get the \ti{log-determinant divergence} (LDD)~\cite{kulis2009low}: $\Gamma_{ldd}(\mb{X},\mb{Y})= \text{tr}(\mb{XY}^{-1})-\log\det (\mb{XY}^{-1})-n$.

In PDML, to encourage orthogonality among the projection vectors (row vectors in $\mb{A}$), Xie et al. \cite{xienear} define a family of regularizers $\Omega_{\phi}(\mb{A})=\Gamma_{\phi}(\mb{AA}^\top,\mb{I})$ which encourage the BMD  between the Gram matrix $\mb{AA}^\top$ and an identity matrix $\mb{I}$ to be small.
$\Omega_{\phi}(\mb{A})$ can be specialized to different instances, based on the choices of $\Gamma_{\phi}(\cdot,\cdot)$. Under SFN, $\Omega_{\phi}(\mb{A})$ becomes $\Omega_{sfn}(\mb{A})=\|\mb{AA}^\top-\mb{I}\|_F^2$, which is used in \cite{wang2012semi,fu2014nokmeans,ge2014graph,chen2017diversity} to promote orthogonality. Under VND, $\Omega_{\phi}(\mb{A})$ becomes $
\Omega_{vnd}(\mb{A})=\text{tr}(\mb{AA}^\top\log (\mb{AA}^\top)-\mb{AA}^\top)+R$. Under LDD, $\Omega_{\phi}(\mb{A})$ becomes $\Omega_{ldd}(\mb{A})=\text{tr}(\mb{AA}^\top)-\log\det (\mb{AA}^\top)-R
$.

\section{Convex Relaxation}
The PDML-BMD problem is non-convex, where the global optimal solution of $\mb{A}$ is very difficult to achieve.
We seek a convex relaxation and solve the relaxed problem instead. The basic idea is to transform PDML into MDML and approximate the BMD regularizers with convex functions.

\subsection{Convex Approximations of the BMD Regularizers}
The approximations are based on the properties of eigenvalues. Given a full-rank matrix $\mb{A}\in\mathbb{R}^{R\times D}$ ($R<D$), we know that $\mb{AA}^\top\in\mathbb{R}^{R\times R}$ is a full-rank matrix with $R$ positive eigenvalues $\lambda_1,\cdots,\lambda_R$ and $\mb{A}^\top \mb{A}\in\mathbb{R}^{D\times D}$ is a rank-deficient matrix with $D-R$ zero eigenvalues and $R$ positive eigenvalues that equal to $\lambda_1,\cdots,\lambda_R$. For a general positive definite matrix $\mb{Z}\in\mathbb{R}^{R\times R}$ whose eigenvalues are $\gamma_1,\cdots,\gamma_R$, we have $\|\mb{Z}\|_{F}^2=\sum_{j=1}^{R}\gamma_j^2$, $\text{tr}(\mb{Z})=\sum_{j=1}^{R}\gamma_j$ and $\log\det \mb{Z}=\sum_{j=1}^{R}\log \gamma_j$.
Next, we leverage these facts to seek convex relaxations of the BMD regularizers.
\paragraph{A convex SFN regularizer}
The eigenvalues of $\mb{AA}^\top-\mb{I}_R$ are $\lambda_1-1,\cdots,\lambda_R-1$ and those of $\mb{A}^\top\mb{A}-\mb{I}_D$ are $\lambda_1-1,\cdots,\lambda_R-1,-1,\cdots,-1$. Then $\|\mb{A}^\top \mb{A}-\mb{I}_D\|_F^2
=\sum_{j=1}^{R}(\lambda_{j}-1)^2+\sum_{j=R+1}^{D}(-1)^2
=\|\mb{AA}^\top-\mb{I}_R\|_F^2+D-R
$. Therefore, the SFN regularizer $\|\mb{AA}^\top-\mb{I}_R\|_F^2$ equals to $\|\mb{A}^\top\mb{A}-\mb{I}_D\|_F^2-D+R=\|\mb{M}-\mb{I}_D\|_F^2-D+R$, where $\mb{M}=\mb{A}^\top\mb{A}$ is a Mahalanobis matrix and $R=\textrm{rank}(\mb{A}^\top \mb{A})=\textrm{rank}(\mb{M})$. It is well-known that the trace norm of a matrix is a convex envelope of its rank~\cite{srebro2005rank}. We use $\text{tr}(\mb{M})$ to approximate $\textrm{rank}(\mb{M})$ and get $\|\mb{AA}^\top-\mb{I}_R\|_F^2\approx\|\mb{M}-\mb{I}_D\|_F^2+\textrm{tr}(\mb{M})-D$, where the right hand side is a convex function. Dropping the constant, we get the convex SFN (CSFN) regularizer defined over $\mb{M}$:
\begin{equation}
\label{eq:cvx_fro}
\widehat{\Omega}_{sfn}(\mb{M})=\|\mb{M}-\mb{I}_D\|_F^2+\textrm{tr}(\mb{M})
\end{equation}
\paragraph{A convex VND regularizer}
Given the eigen-decomposition $\mb{AA}^\top=\mb{U}\bs\Lambda \mb{U}^\top$ where the eigenvalue $\Lambda_{jj}$ equals to $\lambda_j$, based on the property of the matrix logarithm, we have $\log (\mb{AA}^\top)=\mb{U}\widehat{\bs\Lambda} \mb{U}^\top$ where $\widehat{\Lambda}_{jj}=\log\Lambda_{jj}$. Then $(\mb{AA}^\top)\log (\mb{AA}^\top)-(\mb{AA}^\top)=\mb{U}(\bs\Lambda\widehat{\bs\Lambda}-\bs\Lambda)\mb{U}^\top$, where the eigenvalues are $\{\lambda_j\log\lambda_j-\lambda_j\}_{j=1}^{R}$. Then $\Omega_{vnd}(\mb{A})=\sum_{j=1}^R(\lambda_j\log \lambda_j-\lambda_j)+R$. Now we consider a matrix $\mb{A}^\top \mb{A}+\epsilon \mb{I}_D$, where $\epsilon>0$ is a small scalar. Using similar calculation, we have $\Gamma_{vnd}(\mb{A}^\top \mb{A}+\epsilon \mb{I}_D,\mb{I}_D)
=\sum_{j=1}^R((\lambda_j+\epsilon)\log (\lambda_j+\epsilon)-(\lambda_j+\epsilon))+(D-R)(\epsilon\log \epsilon-\epsilon)+D
$.
Performing certain algebra (see Appendix~\ref{sec:approx}), we get $\Omega_{vnd}(\mb{A})\approx \Gamma_{vnd}(\mb{A}^\top \mb{A}+\epsilon \mb{I}_D,\mb{I}_D)+R-D
$. Replacing $\mb{A}^\top \mb{A}$ with $\mb{M}$, approximating $R$ with $\text{tr}(\mb{M})$ and dropping constant $D$, we get the convex VND (CVND) regularizer:
\begin{equation}
\label{eq:cvx_vnm}
\begin{array}{lll}
\widehat{\Omega}_{vnd}(\mb{M})&=&\Gamma_{vnd}(\mb{M}+\epsilon \mb{I}_D,\mb{I}_D)+\text{tr}(\mb{M})\\
&\propto&\text{tr}((\mb{M}+\epsilon \mb{I}_D)\log (\mb{M}+\epsilon \mb{I}_D))
\end{array}
\end{equation}
whose convexity is shown in \cite{nielsen2000quantum}.

\paragraph{A convex LDD regularizer}

We have $
\Omega_{ldd}(\mb{A})=\sum_{j=1}^{R}\lambda_j-\sum_{j=1}^{R}\log\lambda_j-R
$ and $\Gamma_{ldd}(\mb{A}^\top\mb{A}+\epsilon\mb{I}_D,\mb{I}_D)=\sum_{j=1}^{R}\lambda_j+D\epsilon
-(D-R)\log\epsilon-\sum_{j=1}^{R}\log(\lambda_j+\epsilon)$. Certain algebra shows that $\Omega_{ldd}(\mb{A})\approx \Gamma_{ldd}(\mb{A}^\top\mb{A}+\epsilon\mb{I}_D,\mb{I}_D)-(1+\log\epsilon)R +D\log\epsilon$. After replacing $\mb{A}^\top\mb{A}$ with $\mb{M}$, approximating $R$ with $\textrm{tr}(\mb{M})$ and discarding constants, we obtain the convex LDD (CLDD) regularizer:
\begin{equation}
\label{eq:cvx_logdet}
\begin{array}{lll}
\widehat{\Omega}_{ldd}(\mb{M})&=&\Gamma_{ldd}(\mb{M}+\epsilon\mb{I}_D,\mb{I}_D)-(1+\log\epsilon)\textrm{tr}(\mb{M})\\
&\propto&-\textrm{logdet}(\mb{M}+\epsilon\mb{I}_D)+(\log\frac{1}{\epsilon})\textrm{tr}(\mb{M})
\end{array}
\end{equation}
where the convexity of $\textrm{logdet}(\mb{M}+\epsilon\mb{I}_D)$ is proved in \cite{boyd2004convex}. Note that in \cite{davis2007information,qi2009efficient}, an information theoretic regularizer based on log-determinant divergence $\Gamma_{ldd}(\mb{M},\mb{I})=-\textrm{logdet}(\mb{M})+\textrm{tr}(\mb{M})$ is applied to encourage the Mahalanobis matrix to be close to the identity matrix. This regularizer requires $\mb{M}$ to be full rank; in contrast, by associating a large weight $\log\frac{1}{\epsilon}$ to the trace norm $\textrm{tr}(\mb{M})$, our CLDD regularizer encourages $\mb{M}$ to be low-rank. Since $\mb{M}=\mb{A}^\top\mb{A}$, reducing the rank of $\mb{M}$ reduces the number of projection vectors in $\mb{A}$.

We discuss the errors in convex approximation, which are from two sources: one is the approximation of $\Omega_{\phi}(\mb{A})$ using $\Gamma_{\phi}(\mb{A}^\top \mb{A}+\epsilon \mb{I}_D,\mb{I}_D)$ where the error is controlled by $\epsilon$ and can be arbitrarily small (by setting $\epsilon$ to be very small); the other is the approximation of the matrix rank using the trace norm. Though the error of the second approximation can be large, it has been both empirically and theoretically~\cite{candes2012exact} demonstrated that decreasing the trace norm can effectively reduce rank. We empirically verify that 
decreasing the convexified CSFN, CVND and CLDD regularizers can decrease the original non-convex counterparts SFN, VND and LDD (see Appendix~\ref{sec:res}). A rigorous analysis is left for future study.

\subsection{DML with a Convex BMD Regularization}
\label{sec:cvxbmd}
Given these convex BMD (CBMD) regularizers (denoted by $\widehat{\Omega}_{\phi}(\mb{M})$), we relax the non-convex PDML-BMD problems into convex MDML-CBMD formulations by replacing $\|\mb{Ax}-\mb{Ay}\|_2^2$ with $(\mb{x}-\mb{y})^\top\mb{M}(\mb{x}-\mb{y})$ and replacing the non-convex BMD regularizers $\Omega_{\phi}(\mb{A})$ with $\widehat{\Omega}_{\phi}(\mb{M})$:
\begin{equation}
\label{eq:bmd_dml_cvx}
\begin{array}{ll}
\textrm{min}_{\mb{M}\succeq 0}&\frac{1}{|\mathcal{S}|}\sum\limits_{(\mb{x},\mb{y})\in \mathcal{S}}(\mb{x}-\mb{y})^\top\mb{M}(\mb{x}-\mb{y})+\gamma\widehat{\Omega}_{\phi}(\mb{M})+\frac{1}{|\mathcal{D}|} \sum\limits_{(\mb{x},\mb{y})\in \mathcal{D}}\textrm{max}(0, \tau-(\mb{x}-\mb{y})^\top \mb{M} (\mb{x}-\mb{y}))\\
\end{array}
\end{equation}

\section{Optimization}
We use stochastic proximal subgradient descent algorithm~\cite{parikh2014proximal} to solve the MDML-CBMD problems. The algorithm iteratively performs the following steps until convergence: (1) randomly sampling a mini-batch of data pairs, computing the subgradient $\bigtriangleup \mb{M}$ of the data-dependent loss (the first and second term in the objective function) defined on the mini-batch, then performing a subgradient descent update: $\widetilde{\mb{M}}= \mb{M}-\eta \bigtriangleup\mb{M}$, where $\eta$ is a small stepsize; and (2) applying proximal operators associated with the regularizers $\widetilde{\Omega}_{\phi}(\mb{M})$ to $\widetilde{\mb{M}}$. The gradient of the CVND regularizer is $\log(\mb{M}+\epsilon \mb{I}_D)+ \mb{I}_D$. To compute $\log(\mb{M}+\epsilon \mb{I}_D)$, we first perform an eigen-decomposition: $\mb{M}+\epsilon \mb{I}_D=\mb{U}\bs\Lambda\mb{U}^\top$, then take the log of every eigenvalue in $\bs\Lambda$ which gets us a new diagonal matrix $\widetilde{\bs\Lambda}$, and finally compute $\log(\mb{M}+\epsilon \mb{I}_D)$ as $\mb{U}\widetilde{\bs\Lambda}\mb{U}^\top$. In the CLDD regularizer, the gradient of $\textrm{logdet}(\mb{M}+\epsilon\mb{I}_D)$ is $(\mb{M}+\epsilon\mb{I}_D)^{-1}$, which can also be computed by eigen-decomposition. Next, we present the proximal operators.
\subsection{Proximal Operators}
Given the regularizer $\widetilde{\Omega}_{\phi}(\mb{M})$, the associated proximal operator $\text{prox}(\widetilde{\mb{M}})$ is defined as: $\text{prox}(\widetilde{\mb{M}})=\text{argmin}_{\mb{M}}\;\frac{1}{2\eta}\|\mb{M}-\widetilde{\mb{M}}\|_2^2+\gamma\widetilde{\Omega}_{\phi}(\mb{M})$, subject to $\mb{M}\succeq 0$.
Let $\{\tilde{\lambda}_j\}_{j=1}^{D}$ be the eigenvalues of $\widetilde{\mb{M}}$ and $\{x_j\}_{j=1}^{D}$ be the eigenvalues of $\mb{M}$, then the above problem can be equivalently written as:
\begin{equation}
\label{eq:po2}
\begin{array}{ll}
\text{min}_{\{x_j\}_{j=1}^{D}}&\frac{1}{2\eta}\sum\limits_{j=1}^{D}(x_j-\tilde{\lambda}_j)^2+\gamma\sum\limits_{j=1}^{D}h_{\phi}(x_j)\\
s.t.&\forall j=1,\cdots,D, \quad x_j\geq 0
\end{array}
\end{equation}
where $h_{\phi}(x_j)$ is a regularizer-specific scalar function. This problem can be decomposed into $D$ independent ones: (P) $\text{min}_{x_j}\;f(x_j)=\frac{1}{2\eta}(x_j-\tilde{\lambda}_j)^2+\gamma h_{\phi}(x_j)$, subject to $x_j\geq 0$, for $j=1,\cdots,D$, which can be solved individually.

\paragraph{SFN}
For SFN where $\widetilde{\Omega}_{\phi}(\mb{M})=\|\mb{M}-\mb{I}_D\|_F^2+\textrm{tr}(\mb{M})$ and $h_{sfn}(x_j)=(x_j-1)^2+x_j$, the problem (P) is simply a quadratic programming problem. The optimal solution is $x_j^*=\textrm{max}(0,\frac{\tilde{\lambda}_j+\eta\gamma}{1+2\eta\gamma})$

\paragraph{VND}
For VND where $\widetilde{\Omega}_{\phi}(\mb{M})=\text{tr}((\mb{M}+\epsilon \mb{I}_D)\log (\mb{M}+\epsilon \mb{I}_D))$ and $h_{\phi}(x_j)=(x_j+\epsilon)\log(x_j+\epsilon)$, by taking the derivative of the objective function $f(x_j)$ in problem (P) w.r.t $x_j$ and setting the derivative to zero, we get $\eta\gamma\log(x_j+\epsilon)+x_j+\eta\gamma-\tilde{\lambda}_j=0$. The root of this equation is: $\eta\gamma\omega(\frac{\epsilon-\eta\gamma+\tilde{\lambda}_j}{\eta\gamma}-\log(\eta\gamma))-\epsilon$, where $\omega(\cdot)$ is the Wright omega function \cite{gorenflo2007analytical}. If this root is negative, then the optimal $x_j$ is 0; if this root is positive, then the optimal $x_j$ could be either this root or 0. We pick the one that yields the lowest $f(x_j)$. Formally, $x_j^*=\textrm{argmin}_{x_j}\: f(x_j)$, where $x\in\{\textrm{max}(\eta\gamma\omega(\frac{\epsilon-\eta\gamma+\tilde{\lambda}_j}{\eta\gamma}-\log(\eta\gamma))-\epsilon,0),0\}$.

\paragraph{LDD} For LDD where $\widetilde{\Omega}_{\phi}(\mb{M})=-\textrm{logdet}(\mb{M}+\epsilon\mb{I}_D)+(\log\frac{1}{\epsilon})\textrm{tr}(\mb{M})$ and $h_{\phi}(x_j)=-\log(x_j+\epsilon)+x_j\log\frac{1}{\epsilon}$, by taking the derivative of $f(x_j)$ w.r.t $x_j$ and setting the derivative to zero, we get a quadratic equation: $x_j^2+ax_j+b=0$, where $a=\epsilon-\tilde{\lambda}_j-\eta\gamma\log\epsilon$ and $\eta\gamma(1-\epsilon\log\epsilon)$. The optimal solution is achieved either at the positive roots (if any) of this equation or 0. We pick the one that yields the lowest $f(x_j)$. Formally, $x_j^*=\textrm{argmin}_{x_j}\: f(x_j)$, where $x\in\{\textrm{max}(\frac{-b+\sqrt{b^2-4ac}}{2a},0),\textrm{max}(\frac{-b-\sqrt{b^2-4ac}}{2a},0),0\}$.

\paragraph{Computational Complexity} In this algorithm, the major computation workload is eigen-decomposion of $D$-by-$D$ matrices, with a complexity of $O(D^3)$. In our experiments, since $D$ is no more than 1000, $O(D^3)$ is not a big bottleneck. Besides, these matrices are symmetric, the structures of which can thus be leveraged to speed up eigen-decomposition. In implementation, we use the MAGMA\footnote{\url{http://icl.cs.utk.edu/magma/}} library that supports the efficient eigen-decomposition of symmetric matrices on GPU. Note that the unregularized MDML also requires the eigen-decomposition (of $\mb{M}$), hence adding these CBMD regularizes does not substantially increase additional computation cost.

\section{Theoretical Analysis}
In this section, we present theoretical analysis of balancedness and generalization error.

\subsection{Analysis of Balancedness}

In this section, we analyze how the nonconvex BMD regularizers that promote orthogonality affect the balancedness of the distance metrics learned by PDML-BMD\footnote{The analysis of convex BMD regularizers in MDML-CBMD will be left for future work.}. Specifically, the analysis focuses on the following projection matrix: $\mb{A}^*  = \arg\min_{\mb{A}}\;\; \E_{\mc{S}, \mc{D}} [ \frac{1}{|\mathcal{S}|}\sum_{(\mb{x},\mb{y})\in \mathcal{S}}\|\mb{Ax}-\mb{Ay}\|_2^{2}+\frac{1}{|\mathcal{D}|} \sum_{(\mb{x},\mb{y})\in \mathcal{D}}\textrm{max}(0, \tau-\|\mb{Ax}-\mb{Ay}\|_2^{2}) +\gamma\Omega_{\phi}(\mb{A})]$.
We assume there are $K$ classes, where class $k$ has a distribution $p_k$ and the corresponding expectation is $\bs\mu_k$. Each data sample in $\mc{S}$ and $\mc{D}$ is drawn from the distribution of one specific class. We define $\xi_k=\mathbb{E}_{\mb{x}\sim p_k}[\sup_{\lVert \mb{v} \rVert_2 = 1} |\mb{v}^\top (\mb{x}-\bs\mu_k) |]$ and $\xi = \max_k \xi_k$. Further, we assume $\mb{A}^*$ has full rank $R$ (which is the number of the projection vectors), and let $\mb{U}\mb{\Lambda}\mb{U}^\top$ denote the eigen-decomposition of $\mb{A}^*\mb{A}^{*\top}$, where $\mb{\Lambda} = \mathrm{diag}(\lambda_1, \lambda_2, \cdots \lambda_R)$ with $\lambda_1 \ge \lambda_2 \ge \cdots \ge \lambda_R$.

We define an \ti{imbalance factor} (IF) to measure the (im)balancedness. For each class $k$, we use the corresponding expectation $\bs\mu_k$ to characterize this class.

We define the Mahalanobis distance between two classes $j$ and $k$ as: $d_{jk}=(\bs\mu_j - \bs\mu_k)^\top \mb{A}^{*\top} \mb{A}^* (\bs\mu_j - \bs\mu_k)$. We define the IF among all classes as:
\begin{equation}
   \eta = \frac{\max_{j\ne k} d_{jk}}{\min_{j\ne k} d_{jk}}.
\end{equation}
The motivation of such a definition is: for two frequent classes, since they have more training examples and hence contributing more in learning $\mb{A}^*$, DML intends to make their distance $d_{jk}$ large; whereas for two infrequent classes, since they contribute less in learning (and DML is constrained by similar pairs which need to have small distances), their distance may end up being small. Consequently, if classes are imbalanced, some between-class distances can be large while others small, resulting in a large IF. The following theorem shows the upper bounds of IF.
\begin{theorem}\label{thm:inb}
Let $C$ denote the ratio between $\max_{j \ne k} \lVert \bs\mu_j - \bs\mu_k \rVert_2^2$ and $\min_{j\ne k} \lVert \bs\mu_j - \bs\mu_k \rVert_2^2$ and assume $\max_{j,k} \|\bs\mu_j - \bs\mu_k \|_2 \le B_0$. Suppose the regularization parameter $\gamma$ and distance margin $\tau$ are sufficiently large: $\gamma\ge\gamma_0$ and   $\tau\ge\tau_0$, where $\gamma_0$ and $\tau_0$ depend on $\{p_k\}_{k=1}^K$ and $\{\bs\mu_k\}_{k=1}^K$. If $R \ge K - 1$ and $\xi \le (-B_0 + \sqrt{B_0^2 + \lambda_{K-1}\beta_{K-1} / (2\text{tr}(\bs\Lambda)})/4$, then we have the following bounds for the IF\footnote{Please refer to Appendix~\ref{sec:sketch} for the definition of $\beta_{K-1}$ and the detailed proof.}.
\begin{itemize}[leftmargin=*]
   \item For the VND regularizer $\Omega_{vnd}(\mb{A}^*)$, if $\Omega_{vnd}(\mb{A}^*)\leq 1$, the following bound of the IF $\eta$ holds:
$$\eta \le Cg(\Omega_{vnd}(\mb{A}^*))$$
where $g(\cdot)$ is an increasing function defined in the following way. Let $f(c) = c^{1/(c+1)}(1+1/c)$, which is strictly increasing on $(0,1]$ and strictly decreasing on $[1, \infty)$ and let $f^{-1}(c)$ be the inverse function of $f(c)$ on $[1,\infty)$, then $g(c)=f^{-1}(2-c)$ for $c <1$.
\item For the LDD regularizer $\Omega_{ldd}(\mb{A}^*)$, we have
$$\eta \le  4Ce^{\Omega_{ldd}(\mb{A}^*)}$$
\end{itemize}
\end{theorem}
As can be seen, the bounds are increasing functions of the BMD regularizers $\Omega_{vnd}(\mb{A}^*)$ and $\Omega_{ldd}(\mb{A}^*)$. Decreasing these regularizers would reduce the upper bounds of the imbalance factor, hence leading to more balancedness. For SFN, such a bound cannot be derived.

\subsection{Analysis of Generalization Error}\label{subsec:generr}
In this section, we analyze how the convex BMD regularizers affect the generalization error in MDML-CBMD problems. Following~\cite{verma2015sample}, we use \ti{distance-based error} to measure the quality of a Mahalanobis distance matrix $\mb{M}$. Given the sample $\mc{S}$ and $\mc{D}$ where the total number of data pairs is $m=|S|+|D|$, the empirical error is defined as $\widehat{L}(\mb{M}) = \frac{1}{|\mathcal{S}|}\sum_{(\mb{x},\mb{y})\in \mathcal{S}}(\mb{x}-\mb{y})^\top\mb{M}(\mb{x}-\mb{y}) + \frac{1}{|\mathcal{D}|}\sum_{(\mb{x},\mb{y})\in \mathcal{D}}\textrm{max}(0, \tau-(\mb{x}-\mb{y})^\top\mb{M}(\mb{x}-\mb{y}))$ and the expected error is $L(\mb{M}) = \E_{\mc{S}, \mc{D}} [\widehat{L}(\mb{M})]$. Let $\widehat{\mb{M}}^*$ be optimal matrix learned by minimizing the empirical error: $\widehat{\mb{M}}^*=\textrm{argmin}_{\mb{M}} \widehat{L}(\mb{M})$. We are interested in how well $\widehat{\mb{M}}^*$ performs on unseen data. The performance is measured using generalization error: $\mc{E}=L(\widehat{\mb{M}}^*)-\widehat{L}(\widehat{\mb{M}}^*)$. To incorporate the impact of the CBMD regularizers $\Omega_{\phi}(\mb{M})$, we define the hypothesis class of $\mb{M}$ to be $\mc{M} = \{\mb{M}\succeq 0: \Omega_{\phi}(\mb{M}) \le C\}$. The upper bound $C$ controls the strength of regularization. A smaller $C$ entails stronger promotion of orthogonality. $C$ is controlled by the regularization parameter $\gamma$ in Eq.(\ref{eq:bmd_dml_cvx}). Increasing $\gamma$ reduces $C$. 
For different CBMD regularizers, we have the following generalization error bound.  
\begin{theorem}\label{thm:gen vnd}
Suppose $\sup_{\|\mb{v}\|_2\le 1, (\mb{x},\mb{y})\in \mc{S}} |\mb{v}^\top (\mb{x}-\mb{y})| \le B$, then with probability at least $1-\delta$, we have:
\begin{itemize}[leftmargin=*]
    \item For the CVND regularizer, 
    \begin{equation*}
    \begin{array}{l}
\mc{E} \le (4B^2C+\max(\tau, B^2C)\sqrt{2\log(1/\delta)})\frac{1}{\sqrt{m}}.
\end{array}
\end{equation*}
    \item For the CLDD regularizer,
\begin{equation*}
\begin{array}{l}
\mc{E} \le   
(\frac{4B^2C}{\log(1/\epsilon)-1}+\max(\tau,\frac{C-D\epsilon}{\log(1/\epsilon)-1})\sqrt{2\log(1/\delta)})
\frac{1}{\sqrt{m}}.
\end{array}
\end{equation*} 
    \item For the CSFN regularizer,
\begin{align*}
    \begin{array}{l}
\mc{E}\leq  (2B^2\min(2C, \sqrt{C})+ \max(\tau, C) \sqrt{2\log(1/\delta)})\frac{1}{\sqrt{m}}.
\end{array}
\end{align*} 
\end{itemize}
\end{theorem}
From these generalization error bounds (GEBs), we can see two major implications. First, CBMD regularizers can effectively control the GEBs. Increasing the strength of CBMD regularization (by enlarging $\gamma$) reduces $C$, which decreases the GEBs since they are all increasing functions of $C$. Second, the GEBs converge with rate $O(1/\sqrt{m})$, where $m$ is the number of training data pairs. This rate matches with that in \cite{bellet2015robustness,verma2015sample}.

\begin{table}[t]
\scriptsize
\captionsetup{font=footnotesize}
\centering
\caption{ Dataset Statistics}\label{table:retrieval_20news}
\begin{tabular}{lcccc}
\hline
 &\#Train& \#Test& Dim.& \#Class \\
 \hline
MIMIC & 40K& 18K & 1000&2833 \\
EICU &53K&39K&1000&2175\\
Reuters&4152&1779&1000&49\\
News & 11307 & 7538 & 1000 & 20\\
Cars & 8144 & 8041 & 1000& 196 \\
Birds & 9000 & 2788 & 1000& 200 \\
Act& 7352& 2947& 561 &6\\
\hline
\end{tabular}
\label{tb:dt_stats}
\end{table}

\begin{table*}
\scriptsize
\centering
\caption{On the three imbalanced datasets -- MIMIC, EICU, Reuters, we show the mean AUC (averaged on 5 random train/test splits) on all classes (A-All) and infrequent classes (A-IF)  and the balance score. On the rest 4 balanced datasets, A-All is shown. The AUC on frequent classes and the standard errors are in Appendix~\ref{sec:res}.
}
\begin{tabular}{l||ccc|ccc|ccc||c|c|c|c}
\hline
&\multicolumn{3}{c|}{MIMIC}  &\multicolumn{3}{c|}{EICU}  &\multicolumn{3}{c||}{Reuters}&News&Cars&Birds&Act  \\
\cline{2-14}
&A-All & A-IF & BS &A-All & A-IF & BS &A-All & A-IF & BS &A-All&A-All&A-All& A-All \\
\hline
PDML  & 0.634 & 0.608 & 0.070 & 0.671   & 0.637  & 0.077 &0.949 &0.916 &0.049 &0.757&0.714&0.851&  0.949\\
MDML  & 0.641 & 0.617 & 0.064 & 0.677   &  0.652 & 0.055 &0.952 &0.929 &0.034 &0.769&0.722&0.855 &0.952\\
\hline
LMNN & 0.628 & 0.609 & 0.054 &  0.662 & 0.633  & 0.066 &0.943 &0.913 & 0.040 &0.731&0.728&0.832 &0.912\\
LDML &0.619  & 0.594  & 0.068 & 0.667  & 0.647  & 0.046 &0.934 &0.906 &0.042 &0.748&0.706& 0.847&0.937\\
MLEC & 0.621  & 0.605 & 0.045 &0.679   & 0.656  &0.053 &0.927 &0.916 &0.021 &0.761&0.725& 0.814&0.917\\
GMML & 0.607 & 0.588 & 0.053 & 0.668  &  0.648 & 0.045 &0.931 & 0.905&0.035 &0.738&0.707&0.817&0.925\\
ILHD & 0.577 & 0.560 & 0.051 & 0.637  & 0.610  & 0.064&0.905 &0.893 &0.028 &0.711&0.686&0.793&0.898\\
\hline
MDML-$\ell_2$ & 0.648 & 0.627 & 0.055 &0.695 & 0.676  & 0.042  &0.955 &0.930 & 0.037& 0.774&0.728&0.872&0.958 \\
MDML-$\ell_1$ & 0.643 &  0.615& 0.074 & 0.701  & 0.677  &0.053 &0.953 &0.948 & 0.020&0.791&0.725&0.868& 0.961\\
MDML-$\ell_{2,1}$ &0.646  & 0.630 & 0.043 & 0.703  &  0.661 &0.091 &0.963 &0.936 &0.035 &0.783&0.728&0.861&0.964 \\
MDML-Tr & 0.659 & 0.642 & 0.044 & 0.696  & 0.673  & 0.051 &0.961 &0.934 &0.036 &0.785&0.731&0.875&0.955\\
MDML-IT &0.653 & 0.626 & 0.070 &0.692   & 0.668  &0.053 &0.954 &0.920 & 0.046&0.771&0.724& 0.858& 0.967\\
MDML-Drop &0.647  & 0.630 & 0.045 & 0.701& 0.670  & 0.067  &0.959 &0.937 &0.032 & 0.787&0.729&0.864&0.962 \\
MDML-OS &0.649  & 0.626  & 0.059   &      0.689 & 0.679 &  0.045  &    0.957&  0.938&      0.031 &   0.779 & 0.732 &0.869 & 0.963\\
\hline
PDML-DC  & 0.652 & 0.639 & 0.035 & 0.706 & 0.686  & 0.044  & 0.962&0.943 &0.034 & 0.773&0.736&0.882&0.964\\
PDML-CS  & 0.661 & 0.641 & 0.053 &0.712 & 0.670  &  0.089 & 0.967&0.954 &0.020 & 0.803&0.742&0.895&0.971\\
PDML-DPP &0.659 & 0.632 & 0.069 &0.714 &  0.695 & 0.041  &0.958 & 0.937&0.036 & 0.797&0.751&0.891& 0.969\\
PDML-IC  & 0.660  & 0.642 & 0.047 &0.711 & 0.685  &  0.057 & 0.972&0.954 &0.030 & 0.801&0.740&0.887 &0.967\\
PDML-DeC &  0.648 & 0.625 & 0.063 &0.698 & 0.675  & 0.050  &0.965 &0.960 &0.017 &0.786 &0.728& 0.860&0.958\\
PDML-VGF &0.657  & 0.634 & 0.059 & 0.718& 0.697  & 0.045  & 0.974& 0.952&0.036 &0.806 &0.747&0.894 &\tb{0.974}\\
PDML-MA & 0.659  & 0.644 & 0.040 &0.721  & 0.703  & 0.038  &0.975 &0.959 &0.024 &0.815 &0.743&0.898&0.968\\
PDML-OC  &0.651   & 0.636 & 0.041 & 0.705 & 0.685  & 0.043  &0.955 &0.931 & 0.036&0.779 &0.727&0.875&0.956\\
PDML-OS & 0.639&	0.614&	0.067 & 0.675&	0.641 &	0.072 & 0.951& 0.928 &0.038 & 0.764 & 0.716&0.855& 0.950 \\
PDML-SFN  & 0.662  & 0.642 & 0.051 &0.724  & 0.701  & 0.045  & 0.973&0.947 &0.038 & 0.808&0.749&0.896&0.970\\
PDML-VND  &0.667  & 0.655 & 0.032 &0.733& 0.706  & 0.057  &0.976 &0.971 &0.012 & 0.814&0.754&0.902 &0.972\\
PDML-LDD  & 0.664  & 0.651 & 0.035 &0.731& 0.711  & 0.043  &0.973 &0.964 &0.017 & 0.816&0.751&0.904 &0.971\\
\hline
MDML-CSFN  & 0.668 & 0.653 & 0.039 &0.728 & 0.705  & 0.049  &0.978 & 0.968&0.023 & 0.813&0.753&0.905& 0.972\\
MDML-CVND  & \tb{0.672}  & \tb{0.664} & \tb{0.022} & 0.735& 0.718  & \tb{0.035}  & \tb{0.984}& \tb{0.982} &0.012 & \tb{0.822}&0.755&0.908&0.973 \\
MDML-CLDD  & 0.669 & 0.658 & 0.029 & \tb{0.739} & \tb{0.719}  & 0.042  &0.981 & 0.980&\tb{0.011} & 0.819&\tb{0.759}& \tb{0.913} &0.971\\
\hline
\end{tabular}
\label{tb:auc}
\end{table*}

\section{Experiments}
\label{sec:exp}

\paragraph{Datasets}
We used 7 datasets in the experiments: two electronic health record datasets MIMIC (version III)~\cite{johnson2016mimic} and EICU (version 1.1)~\cite{goldberger2000physiobank}; two text datasets Reuters\footnote{\url{http://www.daviddlewis.com/resources/testcollections/reuters21578/}} and 20-Newsgroups (News)\footnote{\url{http://qwone.com/~jason/20Newsgroups/}}; two image datasets Stanford-Cars (Cars) \cite{krause20133d} and Caltech-UCSD-Birds (Birds) \cite{welinder2010caltech}; and one sensory dataset 6-Activities (Act)~\cite{anguita2012human}. The MIMIC-III dataset contains 58K hospital admissions of 47K patients who stayed within the intensive care units (ICU). Each admission has a primary diagnosis (a disease), which acts as the class label of this admission. There are 2833 unique diseases. We extract 7207-dimensional features from demographics, clinical notes, and lab tests. The EICU dataset contains 92K ICU admissions diagnosed with 2175 unique diseases. 3743-dimensional features are extracted from demographics, lab tests, vital signs, and past medical history. For the Reuters datasets, after removing documents that have more than one labels and removing classes that have less than 3 documents, we are left with 5931 documents and 48 classes. Documents in Reuters and News are represented with \ti{tfidf} vectors where the vocabulary size is 5000. For the two image datasets Birds and Cars, we use the VGG16 \cite{simonyan2014very} convolutional neural network trained on the ImageNet \cite{deng2009imagenet} dataset to extract features, which are the 4096-dimensional outputs of the second fully-connected layer. The 6-Activities
dataset contains sensory recordings of 30 subjects performing 6 activities (which are the class labels). The features are 561-dimensional sensory signals. For the first six datasets, the features are normalized using min-max normalization along each dimension and the feature dimension is reduced to 1000 using PCA. Since there is no standard split of the training/test set, we perform five random splits and average the results of the five runs. Dataset statistics are summarized in Table \ref{tb:dt_stats}.
More details of the datasets and feature extraction are deferred to Appendix~\ref{sec:data}.

\paragraph{Experimental Settings}
Two examples are considered as similar if they belong to the same class and dissimilar if otherwise. The learned distance metrics are applied for retrieval (using each test example to query the rest of the test examples) whose performance is evaluated using the Area Under precision-recall Curve (AUC)~\cite{manning2008introduction} which is the higher, the better. Note that the learned distance metrics can also be applied to other tasks such as clustering and classification. Due to the space limit, we focus on retrieval. We apply the proposed convex regularizers CSFN, CVND, CLDD to MDML. We compare them with two sets of baseline regularizers.
The first set aims at promoting orthogonality, which are based on  determinant of covariance (DC) \cite{malkin2008ratio}, cosine similarity (CS) \cite{yu2011diversity}, determinantal point process (DPP) \cite{kulesza2012determinantal,zou2012priors}, InCoherence (IC) \cite{bao2013incoherent}, variational Gram function (VGF)~\cite{zhou2011hierarchical,jalali2015variational}, decorrelation (DeC)~\cite{cogswell2015reducing}, mutual angles (MA)~\cite{xie2015diversifying}, squared Frobenius norm (SFN)~\cite{wang2012semi,fu2014nokmeans,ge2014graph,chen2017diversity}, von Neumann divergence (VND)~\cite{xienear}, log-determinant divergence (LDD)~\cite{xienear}, and orthogonal constraint (OC) $\mb{AA}^\top=\mb{I}$~\cite{liu2008output,wang2015deep}. All these regularizers are applied to PDML. The other set of regularizers are not designed particularly for promoting orthogonality but are commonly used, including $\ell_2$ norm, $\ell_1$ norm~\cite{qi2009efficient}, $\ell_{2,1}$ norm~\cite{ying2009sparse}, trace norm (Tr)~\cite{liu2015low}, information theoretic (IT) regularizer $-\textrm{logdet}(\mb{M})+\textrm{tr}(\mb{M})$~\cite{davis2007information}, and Dropout (Drop)~\cite{srivastava2014dropout}. All these regularizers are applied to MDML. One common way of dealing with class-imbalance is \textit{over-sampling} (OS)~\cite{galar2012review}, which repetitively draws samples from the empirical distributions of infrequent classes until all classes have the same number of samples. We apply this technique to PDML and MDML. In addition, we compare with vanilla Euclidean distance (EUC) and other distance learning methods including large margin nearest neighbor (LMNN) metric learning, information theoretic metric learning (ITML) \cite{davis2007information}, logistic discriminant metric learning (LDML) \cite{guillaumin2009you}, metric learning from equivalence constraints (MLEC)~\cite{kostinger2012large}, geometric mean metric learning (GMML) \cite{zadehgeometric}, and independent Laplacian hashing with diversity (ILHD)~\cite{carreira2016ensemble}. The PDML-based methods except PDML-OC are solved with stochastic subgradient descent (SSD). PDML-OC is solved using the algorithm proposed in~\cite{wen2013feasible}. The MDML-based methods are solved with proximal SSD.
The learning rate is set to 0.001. The mini-batch size is set to 100 (50 similar pairs and 50 dissimilar pairs). We use 5-fold cross validation to tune the regularization parameter among $\{10^{-3},\cdots,10^{0}\}$ and the number of projection vectors (of the PDML methods) among $\{50,100,200,\cdots,500\}$. In CVND and CLDD, $\epsilon$ is set to be $1e-5$. The margin $t$ is set to be 1. In the MDML-based methods, after the Mahalanobis matrix $\mb{M}$ (rank $R$) is learned, we factorize it into $\mb{M}=\mb{L}^\top\mb{L}$ where $\mb{L}\in\mathbb{R}^{R\times D}$ (see Appendix~\ref{sec:appen-set}), then perform retrieval based on $\|\mb{Lx}-\mb{Ly}\|_2^2$, which is more efficient than that based on $(\mb{x}-\mb{y})^\top\mb{M}(\mb{x}-\mb{y})$. Each method is implemented on top of GPU using the MAGMA library. The experiments are conducted on a GPU-cluster with 40 machines.

\begin{table}
\scriptsize
\centering
\caption{ Training time (hours) on seven datasets. The training time of other baseline methods are deferred to Appendix~\ref{sec:res}.}
\begin{tabular}{lcccccccc}
\hline
&MIMIC & EICU & Reuters &News &Cars &Birds  &Act \\
\hline
PDML  & 62.1 &66.6 &5.2 &11.0 &  8.4  &  10.1  &  3.4 \\
MDML   & 3.4 & 3.7 &0.3 & 0.6& 0.5   &  0.6  & 0.2 \\
\hline
PDML-DC  & 424.7  & 499.2&35.2& 65.6&  61.8  &  66.2  & 17.2 \\
PDML-CS  & 263.2 &284.8 & 22.6&47.2&  34.5  &  42.8  &14.4\\
PDML-DPP &  411.8 &479.1 &36.9 &61.9&  64.2  &  70.5  &16.5\\
PDML-IC  & 265.9 &281.2 & 23.4&46.1&  37.5  & 45.2   &15.3\\
PDML-DeC & 458.5 & 497.5& 41.8 &78.2& 78.9   & 80.7  &19.9 \\
PDML-VGF & 267.3 &284.1 & 22.3 &48.9 & 35.8   & 38.7  &15.4 \\
PDML-MA & 271.4 &282.9 & 23.6& 50.2&30.9   &  39.6  &17.5\\
PDML-OC&104.9&118.2&9.6&14.3&14.8&17.0&3.9\\
PDML-SFN & 261.7 & 277.6&22.9 & 46.3&36.2   &  38.2 &15.9\\
PDML-VND   &401.8 &488.3& 33.8&62.5& 67.5   &  73.4 &17.1\\
PDML-LDD   &407.5 & 483.5& 34.3&60.1& 61.8   & 72.6  &17.9\\
\hline
MDML-CSFN   &41.1 &43.9 &3.3 &7.3 &6.5    & 6.9  &1.8\\
MDML-CVND    &43.8 & 46.2 & 3.6& 8.1& 6.9&7.8  &2.0    \\
MDML-CLDD   & 41.7&44.5 & 3.4& 7.5& 6.6  & 7.2  & 1.8\\
\hline
\end{tabular}
\label{tb:rt}
\end{table}

\paragraph{Results}
The training time taken by different methods to reach convergence is shown in Table \ref{tb:rt}. For the non-convex, PDML-based methods, we report the total time taken by the following computation: tuning the regularization parameter (4 choices) and the number of projection vectors (NPVs, 6 choices) on a two-dimensional grid via 3-fold cross validation ($4\times 6\times3=72$ experiments in total); for each of the 72 experiments, the algorithm restarts 5 times\footnote{Our experiments show that for non-convex methods, multiple re-starts are of great necessity to improve performance. For example, for PDML-VND on MIMIC with 100 projection vectors, the AUC is non-decreasing with the number of re-starts: the AUC after $1,2,...,5$ re-starts are 0.651, 0.651, 0.658, 0.667, 0.667.}, each with a different initialization, and picks the one yielding the lowest objective value. In total, the number of runs is $72\times 5=360$. For the MDML-based methods, there is no need to restart multiple times or tune the NPVs. The total number of runs is $4\times 3=12$. As can be seen from the table, the proposed convex methods are much faster than the non-convex ones, due to the greatly reduced number of experimental runs, although for each single run the convex methods are less efficient than the non-convex methods due to the overhead of eigen-decomposition. The unregularized MDML takes the least time to train since it has no parameters to tune and runs only once. On average, the time of each single run in MDML-(CSFN,CVND,CLDD) is close to that in the unregularized MDML, since an eigen-decomposition is required anyway regardless of the presence of the regularizers.

\begin{table*}
\scriptsize
\centering
\captionsetup{font=footnotesize}
\caption{ Number of projection vectors (NPV) and compactness score (CS,$\times 10^{-3}$). }
\begin{tabular}{l|cc|cc|cc|cc|cc|cc|cc}
\hline
&\multicolumn{2}{c|}{MIMIC}& \multicolumn{2}{c|}{EICU} & \multicolumn{2}{c|}{Reuters}&\multicolumn{2}{c|}{News}  &\multicolumn{2}{c|}{Cars}& 	\multicolumn{2}{c|}{Birds}&\multicolumn{2}{c}{Act} 		\\
\hline
 &NPV&CS&NPV&CS&NPV&CS&NPV&CS&NPV&CS&NPV&CS&NPV&CS\\
\hline
PDML &300&2.1&400&1.7&300&3.2&300&2.5&300&2.4&500&1.7&200&4.7\\
MDML &247&2.6&318&2.1&406&2.3&336&2.3&376&1.9&411&2.1&168&5.7\\
\hline
LMNN &200&3.1&400&1.7&400&2.4&300&2.4&400&1.8&500&1.7&300&3.0\\
LDML &300&2.1&400&1.7&400&2.3&200&3.7&300&2.4&400&2.1&300&3.1\\
MLEC &487&1.3&493&1.4&276&3.4&549&1.4&624&1.2&438&1.9&327&2.8\\
GMML &1000&0.6&1000&0.7&1000&0.9&1000&0.7&1000&0.7&1000&0.8&1000&0.9\\
ILHD &100&5.8&100&6.4&\tb{50}&18.1&100&7.1&100&6.9&100&7.9&\tb{50}&18.0\\
\hline
MDML-$\ell_2$&269&2.4&369&1.9&374&2.6&325&2.4&332&2.2&459&1.9&179&5.4\\
MDML-$\ell_1$ &341&1.9&353&2.0&417&2.3&317&2.5&278&2.6&535&1.6&161&6.0\\
MDML-$\ell_{2,1}$ &196&3.3&251&2.8&288&3.3&316&2.5&293&2.5&326&2.6&135&7.1\\
MDML-Tr &148&4.5&233&3.0&217&4.4&254&3.1&114&6.4&286&3.1&129&7.4\\
MDML-IT &1000&0.7&1000&0.7&1000&1.0&1000&0.8&1000&0.7&1000&0.9&1000&1.0\\
MDML-Drop &183&3.5&284&2.5&315&3.0&251&3.1&238&3.1&304&2.8&147&6.5\\
\hline
PDML-DC &	100&6.5&300&2.4&100&9.6&200&3.9&	200&3.7&	300&2.9&100&9.6\\
PDML-CS &	200&3.3&200&3.6&200&4.8&100&8.0&100&7.4	&200 &4.5&\tb{50}& \tb{19.4}	\\
PDML-DPP &100&6.6&200&3.6&100&9.6&100&8.0	&200&3.8	&	200&4.5&100&9.7\\
PDML-IC &	200&3.3& 200 &3.6&200 &4.9&100 &8.0&200&3.7&	100&8.9&100&9.7\\
PDML-DeC &200&3.2&300&2.3&200&4.8&200	&3.9&200	&3.6&	200&4.3&100&9.6\\
PDML-VGF 	&200&3.3 &200&3.6&200&4.9&100&8.1 &200	&3.7& 200&4.5&100&	9.7\\
PDML-MA 	&200&3.3&200&3.6&100&9.8&100&8.2&100	&7.4&200&4.5&\tb{50}&	\tb{19.4}\\
PDML-SFN &100&6.6&200&3.6&100&9.7&100&8.1	&100&7.5	& 200&4.5&\tb{50}&	\tb{19.4}\\
PDML-OC &100&6.5&100&7.1&\tb{50}&19.1&50&15.6	&100&7.3	&100 &8.8&\tb{50}&	19.1\\
\hline
\hline
PDML-VND&100&6.7	&100&7.3&\tb{50}&\tb{19.5}&100&8.1&	100&7.5&100&9.0 &\tb{50}	&\tb{19.4}\\
PDML-LDD& 100&6.6&200&3.7&100&9.7&100&	8.2&100&7.5	&100&9.0 &\tb{50}&	\tb{19.4}\\
\hline
MDML-CSFN& 143&4.7&	209&3.5&174&5.6&87&9.3&\tb{62}&\tb{12.1}	& 139&6.5&64&	15.2\\
MDML-CVND&	\tb{53}&\tb{12.7}&\tb{65}&\tb{11.3}&61&16.0&63&13.0&127	&5.9& 92&9.9&68&	14.3\\
MDML-CLDD&	76&8.8&128&5.8&85&11.5&\tb{48}&\tb{17.1}&91&8.3	& \tb{71}&\tb{12.9}&55&	17.7\\
\hline
\end{tabular}
\label{tb:num_com}
\end{table*}

Next, we verify whether CSFN, CVND and CLDD are able to learn more balanced distance metrics. On three datasets MIMIC, EICU and Reuters where the classes are imbalanced, we consider a class as ``frequent'' if it contains more than 1000 examples, and ``infrequent'' if otherwise. 
We measure AUCs on all classes (A-All), infrequent classes (A-IF) and frequent classes (A-F), then define a \ti{balance score} (BS) as $|\frac{\textrm{A-IF}}{\textrm{A-F}}-1|$. A smaller BS indicates more balancedness. As shown in Table \ref{tb:auc}, MDML-(CSFN,CVND,CLDD) achieve the highest A-All on 6 datasets and the highest A-IF on all 3 imbalanced datasets. In terms of BS, our convex methods outperform all baseline DML methods.  
These results demonstrate our methods can learn more balanced metrics. By encouraging the projection vectors to be close to being orthogonal, our methods can reduce the redundancy among vectors. Mutually complementary vectors can achieve a broader coverage of latent features, including those associated with infrequent classes. As a result, these vectors improve the performance on infrequent classes and lead to better balancedness. Thanks to their convexity nature, our methods can achieve the global optimal solution and outperform the non-convex ones that can only achieve a local optimal and hence a sub-optimal solution. Comparing (PDML,MDML)-OS with the unregularized PDLM/MDML, we can see that over-sampling indeed improves balancedness. However, this improvement is less significant than that achieved by our methods. 
In general, the orthogonality-promoting (OP) regularizers outperform the non-OP regularizers, suggesting the effectiveness of promoting orthogonality. The orthogonal constraint (OC)~\cite{liu2008output,wang2015deep} imposes strict orthogonality, which may be too restrictive that hurts performance. ILHD~\cite{carreira2016ensemble} learns binary hash codes, which makes retrieval more efficient, however, it achieves lower AUCs due to the quantization errors. MDML-(CSFN,CVND,CLDD) outperform popular DML approaches including LMNN, LDML, MLEC and GMML, demonstrating their competitive standing in the DML literature.

Next we verify whether the learned distance metrics by MDML-(CSFN,CVND,CLDD) are compact. Table \ref{tb:num_com} shows the numbers of the projection vectors (NPVs) that achieve the AUCs in Table \ref{tb:auc}. For MDML-based methods, the NPV equals to the rank of the Mahalanobis matrix since $\mb{M}=\mb{A}^\top\mb{A}$. We define a \ti{compactness score} (CS) which is the ratio between A-All (given in Table~\ref{tb:auc}) and NPV. A higher CS indicates achieving higher AUC by using fewer projection vectors. From Table \ref{tb:num_com}, we can see that on 5 datasets, MDML-(CSFN,CVND,CLDD) achieve larger CSs than the baseline methods, demonstrating their better capability in learning compact distance metrics. Similar to the observations in Table \ref{tb:auc}, CSFN, CVND and CLDD perform better than non-convex regularizers, and CVND, CLDD perform better than CSFN. The reduction of NPV improves the efficiency of retrieval since the computational complexity grows linearly with this number. Together, these results demonstrate that MDML-(CSFN,CVND,CLDD) outperform other methods in terms of learning both compact and balanced distance metrics.

As can be seen from Table~\ref{tb:auc}, our methods MDML-(CVND,CLDD) achieve the best AUC-All. In Table \ref{tb:gap} (Appendix~\ref{sec:res}), it is shown that MDML-(CVND,CLDD) have the smallest gap between training and testing AUC. This indicates that our methods are better capable of reducing overfitting and improving generalization performance.

\section{Conclusions}
In this paper, we have attempted to address three issues of existing orthogonality-promoting DML methods, which include computational inefficiency and lacking theoretical analysis in balancedness and generalization. To address the computation issue, we perform a convex relaxation of these regularizers and develop a proximal gradient descent algorithm to solve the convex problems. To address the analysis issue, we define an imbalance factor (IF) to measure (im)balancedness and prove that decreasing the Bregman matrix divergence regularizers (which promote orthogonality) can reduce the upper bound of the IF, hence leading to more balancedness. We provide a generalization error (GE) analysis showing that decreasing the convex regularizers can reduce the GE upper bound. Experiments on datasets from different domains demonstrate that our methods are computationally more efficient and are more capable of learning balanced, compact and generalizable distance metrics than other approaches.

\appendix

\section{Convex Approximations of BMD Regularizers}
\label{sec:approx}

\paragraph{Approximation of VND regularizer}
Given $\mb{AA}^\top=\mb{U}\bs\Lambda \mb{U}^\top$, according to the property of matrix logarithm, $\log (\mb{AA}^\top)=\mb{U}\widehat{\bs\Lambda} \mb{U}^\top$, where $\widehat{\Lambda}_{jj}=\log\lambda_j$. Then $(\mb{AA}^\top)\log (\mb{AA}^\top)-(\mb{AA}^\top)=\mb{U}(\bs\Lambda\widehat{\bs\Lambda}-\bs\Lambda)\mb{U}^\top$, where the eigenvalues are $\{\lambda_j\log\lambda_j-\lambda_j\}_{j=1}^{R}$. Since $\text{tr}(\mb{M})=\sum_{j=1}^{R}\lambda_j$, we have $\Omega_{vnd}(\mb{A})=\sum_{j=1}^R(\lambda_j\log \lambda_j-\lambda_j)+R$. Now we consider a matrix $\mb{A}^\top \mb{A}+\epsilon \mb{I}_D$, where $\epsilon>0$ is a small scalar. The eigenvalues of this matrix are $\lambda_1+\epsilon,\cdots, \lambda_R+\epsilon,\epsilon,\cdots,\epsilon$. Then we have
\begin{equation}
\begin{array}{ll}
\Gamma_{vnd}(\mb{A}^\top \mb{A}+\epsilon \mb{I}_D, \mb{I}_D)\\
=\text{tr}((\mb{A}^\top \mb{A}+\epsilon \mb{I}_D)\log (\mb{A}^\top \mb{A}+\epsilon \mb{I}_D)-(\mb{A}^\top \mb{A}+\epsilon \mb{I}_D))+D\\
=\sum_{j=1}^R((\lambda_j+\epsilon)\log (\lambda_j+\epsilon)-(\lambda_j+\epsilon))
+\sum_{j=R+1}^D(\epsilon\log \epsilon-\epsilon)+D\\
=\sum_{j=1}^R((\lambda_j+\epsilon)(\log\lambda_j+\log(1+\frac{\epsilon}{\lambda_j}))-(\lambda_j+\epsilon))
+(D-R)(\epsilon\log \epsilon-\epsilon)+D\\
=\sum_{j=1}^R(\lambda_j\log\lambda_j-\lambda_j+\lambda_j\log(1+\frac{\epsilon}{\lambda_j})
+\epsilon(\log\lambda_j+\log(1+\frac{\epsilon}{\lambda_j}))-\epsilon)
+(D-R)(\epsilon\log \epsilon-\epsilon)+D\\
=\Omega_{vnd}(\mb{A})-R+\sum_{j=1}^R(\lambda_j\log(1+\frac{\epsilon}{\lambda_j})
+\epsilon(\log\lambda_j+\log(1+\frac{\epsilon}{\lambda_j}))-\epsilon)
+(D-R)(\epsilon\log \epsilon-\epsilon)+D\\
\end{array}
\end{equation}

Since $\epsilon$ is small, we have $\log(1+\frac{\epsilon}{\lambda_j})\approx \frac{\epsilon}{\lambda_j}$. Then $\lambda_j\log(1+\frac{\epsilon}{\lambda_j})\approx \epsilon$ and the last line in the above equation can be approximated with $\Omega_{vnd}(\mb{A})-R+D+O(\epsilon)$, and therefore
\begin{equation}
\begin{array}{ll}
\Omega_{vnd}(\mb{A})\approx \Gamma_{vnd}(\mb{A}^\top \mb{A}+\epsilon \mb{I}_D,\mb{I}_D)+R-D
\end{array}
\end{equation}
where $O(\epsilon)$ is small since $\epsilon$ is small, and is hence dropped.

\paragraph{Approximation of LDD regularizer}
\begin{equation}
\begin{array}{l}
\Gamma_{ldd}(\mb{A}^\top \mb{A}+\epsilon \mb{I}_D,\mb{I}_D)\\
=\sum\limits_{j=1}^{R}\lambda_j +D\epsilon -(D-R)\log\epsilon-\sum\limits_{j=1}^{R}\log(\lambda_j+\epsilon)\\
=\sum\limits_{j=1}^{R}\lambda_j +D\epsilon -(D-R)\log\epsilon-\sum\limits_{j=1}^{R}(\log\lambda_j+\log(1+\frac{\epsilon}{\lambda_j}))\\
\approx\sum\limits_{j=1}^{R}(\lambda_j-\log\lambda_j)+R\log\epsilon-\epsilon\sum\limits_{j=1}^{R}\frac{1}{\lambda_j}+D\epsilon-D\log\epsilon\\
=\Omega_{ldd}(\mb{A})+R+R\log\epsilon+O(\epsilon)-D\log\epsilon
\end{array}
\end{equation}
Dropping $O(\epsilon)$, we obtain
\begin{equation}
\Omega_{ldd}(\mb{A})=\Gamma_{ldd}(\mb{A}^\top \mb{A}+\epsilon \mb{I}_D,\mb{I}_D)-(\log\epsilon+1)R+D\log\epsilon
\end{equation}

\section{Proof of Theorem \ref{thm:inb}}
\subsection{Proof Sketch}
\label{sec:sketch}

We make the following two assumptions. 
\begin{itemize}
\item The size of similar and dissimilar set $|\mc{S}|$ and $|\mc{D}|$ are fixed.
\item $\mb{A}^*$ has full row rank $R$.
\end{itemize}
Denote the $K$ classes as $\mathcal{C}_1, \mathcal{C}_2, \cdots \mathcal{C}_K$. The probability that a sample is drawn from the $k$th class is $p_k$, and $\sum_{k=1}^K p_k = 1$. Denote the class membership of example $\mb{x}$ as $c(\mb{x})$. Denote the probability that $\mb{x}\in \mc{C}_j, \mb{y} \in \mc{C}_k$ where $(\mb{x}, \mb{y}) \in \mc{D}$ as $p_{jk} = p_j p_k / (1 - \sum_{l=1}^K p_l^2)$.
Define the SVD of matrix $\mb{A}^*$ as $\mb{U}\sqrt{\bs\Lambda}\mb{V}^\top$ where $\mb{U}\in\mathbb{R}^{R\times R}$, $\bs\Lambda\in\mathbb{R}^{R\times R}$, and $\mb{V}\in\mathbb{R}^{D\times R}$. $\bs\Lambda = \mathrm{diag}(\lambda_1, \lambda_2, \cdots \lambda_R)$. then $\mb{A}^{*\top}\mb{A}^* = \mb{V}\bs\Lambda\mb{V}^\top$. Denote $\mb{V} = [\mb{v}_1, \mb{v}_2, \cdots \mb{v}_R]$. Then $\forall \mb{z} = \mb{x} - \mb{y}$, we have $\mb{z}^\top \mb{A}^{*\top} \mb{A}^* \mb{z} = \sum_{r=1}^R \lambda_r (\mb{v}_r^\top \mb{z})^2$. We see $\mb{z}^\top \mb{A}^{*\top} \mb{A}^* \mb{z}$ can be written as a sum of $R$ terms. Inspired by this, we define a vector function $\alpha(\cdot)$ as $\alpha(\mb{u}) = \sum_{j,k=1}^K p_{jk} (\mb{u}^\top (\bs\mu_j-\bs\mu_k))^2$. This function measures the weighted sum of $(\mb{u}^\top (\bs\mu_j - \bs\mu_k))^2$ across all classes. Define $\mc{G} = \spann\{\bs\mu_j - \bs\mu_k: j \ne k\}$. 

\begin{definition} [feature values and feature vectors]
For a linear space $\mc{G}$, define vectors $\mb{w}_1, \mb{w}_2, \cdots \mb{w}_{K-1}$ and positive real numbers $\beta_1, \beta_2, \cdots \beta_{K-1}$ as
$$\mb{w}_1 = \arg \min_{\lVert \mb{u} \rVert=1, \mb{u}\in \mc{G}} \alpha(\mb{u}), \ \ \beta_1 = \alpha(\mb{w}_1),$$
$$\mb{w}_r = \arg \min_{
\scriptsize\begin{array}{c}
\lVert \mb{u} \rVert=1, \mb{u}\in \mc{G}\\
\mb{u}\perp \mb{w}_j, \forall j < r
\end{array}
} \alpha(\mb{u}), \ \ \beta_r = \alpha(\mb{w}_r), \ \ \forall r > 1$$
$\forall r >{K-1}$, define $\beta_r=0$, and $\mb{w}_r$ as an arbitrary vector which has norm 1 and is orthogonal to $\mb{w}_1, \mb{w}_2, \cdots \mb{w}_{r-1}$. $\mb{w}_1, \mb{w}_2, \cdots $ are called feature vectors of $\mathcal{G}$, and $\beta_1, \beta_2, \cdots$ are called feature values of $\mathcal{G}$. 
\end{definition}

We give a condition for the regularizers. 
\begin{condition}\label{cond:invariant}
For a regularizer $\Omega_\phi(\cdot)$, there exists a unique matrix function $\varphi(\cdot)$ such that for any $\mb{A}^*$, 
$$\Omega_\phi(\mb{A}^*) = \varphi(\mb{A}^*\mb{A}^{*\top}) = \varphi(\bs\Lambda).$$
\end{condition}
The VND and LDD regularizer satisfy this condition. For the VND regularizer, $\varphi(\bs\Lambda) = \tr(\bs\Lambda \log \bs\Lambda - \bs\Lambda) + R$; for the LDD regularizer, $\varphi(\bs\Lambda) = \tr(\bs\Lambda) - \log \det (\bs\Lambda) - R$. The SFN regularizer does not satisfy this condition. 

Now we have enough preparation to give the following lemma. It shows that the linear space $\mc{G}$ can be recovered if the second moment of noise is smaller than a certain value.

\begin{lemma}\label{thm:recover feature space}
Suppose $R \ge K - 1$, $\max_{j\in k} \|\bs\mu_j - \bs\mu_k \|_2 \le B_0$, and the regularization parameter $\gamma$ and distance margin $\tau$ satisfy $\gamma\ge \gamma_0, \tau \ge \tau_0$. If $\xi \le \frac{-B_0 + \sqrt{B_0^2 + \gamma_{K-1}\beta_{K-1} / (2\tr(\bs\Lambda))}}{4}$, then
\begin{equation}\label{eq:feature space}
\mc{G} \subset \spann(\mb{A}^{*\top}).
\end{equation}
Here $\spann(\mb{A}^{*\top})$ denotes the column space of matrix $\mb{A}^{*\top}$. Both $\lambda_0$ and $\tau_0$ depend on $\bs\mu_1, \bs\mu_2, \cdots \bs\mu_K$ and $p_1, p_2, \cdots p_K$.
\end{lemma}

The next lemma shows that if Eq.(\ref{eq:feature space}) holds, we can bound the imbalance factor $\eta$ with the condition number of $\mb{A}^*\mb{A}^{*\top}$ (denoted by $\cond(\mb{A}^*\mb{A}^{*\top})$). Note that the BMD regularizers $\Omega_\phi(\mb{A}^*)$ encourage $\mb{A}^*\mb{A}^{*\top}$ to be close to an identity matrix, i.e., encouraging the condition number to be close to 1.
\begin{lemma}\label{thm:inb}
If Eq.(\ref{eq:feature space}) holds, and there exists a real function $g$ such that
$$\cond(\mb{A}^*\mb{A}^{*\top}) \le g(\Omega_\phi(\mb{A}^*)),$$
then we have the following bound for the imbalance factor
$$\eta \le g(\Omega_\phi(\mb{A}^*))\frac{\max_{j \ne k} \lVert \bs\mu_j - \bs\mu_k \rVert^2}{\min_{j\ne k} \lVert \bs\mu_j - \bs\mu_k \rVert^2}.$$
\end{lemma}

Next, we derive the explicit forms of $g$ for the VND and LDD regularizers. 
\begin{lemma}\label{thm:cond regu}
For the VND regularizer $\Omega_{vnd}(\mb{A}^*)$, define $f(c) = c^{1/(c+1)}(1+1/c)$, then $f(c)$ is strictly increasing on $(0,1]$ and strictly decreasing on $[1, \infty)$. Define the inverse function of $f(\cdot)$ on $[1,\infty)$ as $f^{-1}(\cdot)$. Then if $\Omega_{vnd}(\mb{A}^*) <1$, we have
$$\cond(\mb{A}^*\mb{A}^{*\top}) \le f^{-1}(2-\Omega_{vnd}(\mb{A}^*)).$$
For the LDD regularizer $\Omega_{ldd}(\mb{A}^*)$, we have
$$\cond(\mb{A}^*\mb{A}^{*\top}) \le 4e^{\Omega_{ldd}(\mb{A}^*)}.$$
\end{lemma}

Combining Lemma \ref{thm:recover feature space}, 
\ref{thm:inb} and \ref{thm:cond regu}, we finish the proof of Theorem \ref{thm:inb}.

\subsection{Proof of Lemma \ref{thm:recover feature space}}
In order to prove Lemma \ref{thm:recover feature space}, we first need some auxiliary lemmas on the properties of the function $\alpha(\cdot)$. Denote $\bs\mu_{jk} = \bs\mu_j - \bs\mu_k, \forall j \ne k$.

\begin{lemma}\label{lem:rotation invariant}
Suppose $\mb{u}_1, \mb{u}_2, \cdots \mb{u}_r$ and $\mb{v}_1, \mb{v}_2, \cdots \mb{v}_r$ are two sets of standard orthogonal vectors in $\mathbb{R}^d$, and $\spann(\mb{u}_1, \mb{u}_2, \cdots \mb{u}_r) = \spann(\mb{v}_1, \mb{v}_2, \cdots \mb{v}_r)$, then we have
$$\sum_{l=1}^r \alpha(\mb{u}_l) = \sum_{l=1}^r \alpha(\mb{v}_l).$$
\end{lemma}
\begin{proof}
By the definition of these two sets of vectors, there exists a $r\times r$ standard orthogonal matrix $\mb{B}= (b_{jk})$, such that $(\mb{u}_1, \mb{u}_2, \cdots \mb{u}_r) = (\mb{v}_1, \mb{v}_2, \cdots \mb{v}_r)\mb{B}$. Then we have
\begin{align*}
 \sum_{l=1}^r \alpha(\mb{u}_l) = & \sum_{l=1}^r \sum_{j\ne k} p_{jk} ((\sum_{s=1}^r b_{ls} \mb{v}_s)^\top \bs\mu_{jk})^2 \\
= & \sum_{l=1}^r \sum_{j\ne k} p_{jk} \sum_{s, t=1}^r b_{ls}b_{lt} \mb{v}_s^\top \bs\mu_{jk} \mb{v}_t^\top \bs\mu_{jk} \\
= & \sum_{s=1}^r \sum_{j\ne k} p_{jk} (\mb{v}_s^\top \bs\mu_{jk})^2 \sum_{l=1}^r b_{ls}^2 + \sum_{j\ne k} p_{jk} \sum_{s, t=1}^r \ \mb{v}_s^\top \bs\mu_{jk} \mb{v}_t^\top \bs\mu_{jk} \sum_{l=1}^r b_{ls}b_{lt}\\
\end{align*}
Since $\mb{B}$ is a standard orthogonal matrix, we have $\forall s$, $\sum_{l=1}^r b_{ls}^2 = 1$ and $\forall s \ne t$, $\sum_{l=1}^r b_{ls}b_{lt} = 0$. Further, we have 
$$\sum_{l=1}^r \alpha(\mb{u}_l) = \sum_{l=1}^r \alpha(\mb{v}_l).$$
\end{proof}

\begin{lemma}\label{lem:feature vector max}
For any positive integer $r$, any set of standard orthogonal vectors $\mb{u}_1, \mb{u}_2, \cdots \mb{u}_r \in \mathbb{R}^d$, and real numbers $\gamma_1 \ge \gamma_2 \ge \cdots \ge \gamma_r \ge 0$, we have 
\begin{equation}
\sum_{l=1}^r \gamma_l\alpha(\mb{u}_l) \le \sum_{l=1}^r \gamma_l \beta_l,
\end{equation}
where $\beta_l$ is the $l$-th feature value.
\end{lemma}
\begin{proof}
We first prove the situation that $\gamma_1 = \gamma_2 = \cdots =\gamma_r = 1$, i.e.,
\begin{equation}\label{eq:feature value}
\sum_{l=1}^r \alpha(\mb{u}_l) \le \sum_{l=1}^r \beta_l.
\end{equation}
We prove it by induction on $r$. For $r=1$, by the definition of feature values and feature vectors, Eq.(\ref{eq:feature value}) holds. Now supposing Eq.(\ref{eq:feature value}) holds for $r=s$, we prove it holds for $r=s+1$ by contradiction. If Eq.(\ref{eq:feature value}) does not hold, then there exist standard orthogonal vectors $\mb{u}_1, \mb{u}_2, \cdots \mb{u}_{s+1} \in \mathbb{R}^d$, such that
\begin{equation}\label{space>}
\sum_{l=1}^{s+1} \alpha(\mb{u}_l) > \sum_{l=1}^{s+1} \alpha(\mb{w}_l),
\end{equation}
where $\mb{w}_l$ are feature vectors. Since the dimension of $\spann(\mb{u}_1, \mb{u}_2, \cdots \mb{u}_{s+1})$ is $s+1$, there exists $\tilde{\mb{w}}_{s+1} \in \spann(\mb{u}_1, \mb{u}_2, \cdots \mb{u}_{s+1})$, such that $\tilde{\mb{w}}_{s+1} \perp \mb{w}_l, \forall 1\le l \le s$. By the definition of feature vector $\mb{w}_{s+1}$, we have
\begin{equation}\label{space>=}
\sum_{l=1}^{s+1} \alpha(\mb{w}_l) \ge \sum_{l=1}^s \alpha(\mb{w}_l) + \alpha(\tilde{\mb{w}}_{s+1}).
\end{equation}
Let $\tilde{\mb{w}}_1, \tilde{\mb{w}}_2, \cdots \tilde{\mb{w}}_{s+1}$ be a set of standard orthogonal basis of $\spann(\mb{u}_1, \mb{u}_2, \cdots \mb{u}_{s+1})$, by Lemma \ref{lem:rotation invariant}, we have
\begin{equation}\label{space=}
\sum_{l=1}^{s+1} \alpha(\mb{u}_l) = \sum_{l=1}^{s+1} \alpha(\tilde{\mb{w}}_l).
\end{equation}
Combine equation (\ref{space>}), (\ref{space>=}) and (\ref{space=}) we get
$$\sum_{l=1}^{s+1} \alpha(\tilde{\mb{w}}_l) > \sum_{l=1}^{s} \alpha(\mb{w}_l) + \alpha(\tilde{\mb{w}}_{s+1}).$$
Thus we have
$$\sum_{l=1}^{s} \alpha(\tilde{\mb{w}}_l) > \sum_{l=1}^{s} \alpha(\mb{w}_l).$$
This contradicts with our induction assumption. The proof for the $\gamma_1 = \gamma_2 = \cdots =\gamma_r = 1$ case completes.

Next, we prove the situation that $\gamma_l$ are not all equal to 1, by utilizing Eq.(\ref{eq:feature value}). 
\begin{align*}
\sum_{l=1}^r \gamma_l\alpha(\mb{u}_l) & = \sum_{l=1}^{r-1} [(\gamma_l - \gamma_{l+1}) \sum_{t=1}^l \alpha(\mb{u}_t)] + \gamma_r \sum_{t=1}^r \alpha(\mb{u}_t) \\
	& \le \sum_{l=1}^{r-1} [(\gamma_l - \gamma_{l+1}) \sum_{t=1}^l \beta_t] + \gamma_r \sum_{t=1}^r \beta_t \\
	& \le \sum_{l=1}^r \gamma_l \beta_l
\end{align*}
The proof completes. 
\end{proof}
Note that in Lemma \ref{lem:feature vector max}, $r$ can be larger than the number of nonzero feature values $K-1$. This will be used in the proof of Lemma \ref{thm:recover feature space} later.

Another auxiliary lemma needed to prove Lemma \ref{thm:recover feature space} is given below.  
\begin{lemma}\label{lem:constrained feature values}
Suppose $\mb{w}_0\in \mc{G}$, define linear space $\mathcal{H} = \{\mb{v}\in \mathcal{G}: \mb{v} \perp \mb{w}_0 \}$. Then there are $K-2$ nonzero feature values of $\mathcal{H}$. Denote them as $\beta'_1, \beta'_2, \cdots \beta'_{K-2}$, then $\forall \ r\le K-2$, $\forall \ \gamma_1\ge \gamma_2 \ge \cdots \ge \gamma_{r} \ge 0$,
$$\sum_{l=1}^r \gamma_l \beta'_l \le \sum_{l=1}^r \gamma_l\beta_l$$
\end{lemma}
\begin{proof}
Note that the dimension of $\mathcal{H}$ is $K-2$, then there are $K-2$ nonzero feature values. The feature vectors of $\mathcal{H}$ are also standard orthogonal vectors of the linear space $\mc{G}$. By Lemma \ref{lem:feature vector max}, we have $\sum_{l=1}^r \gamma_l \beta'_l \le \sum_{l=1}^r \gamma_l\beta_l, \ \forall \ r \le K-2$.
\end{proof}
Now we are ready to prove Lemma \ref{thm:recover feature space}.

\begin{proof}(of Lemma \ref{thm:recover feature space}) We conduct the proof by contradiction. Assuming Eq.(\ref{eq:feature space}) does not hold, we prove $\mb{A}^*$ can not be the global optimal solution of PDML. Let $\mb{U}\sqrt{\bs\Lambda}\mb{V}^\top$ be the SVD of $\mb{A}^*$. Define $\mb{W} = (\mb{w}_1, \mb{w}_2, \cdots \mb{w}_R)$ as a matrix whose columns contain the feature vectors. Let $\tilde{\mb{A}} = \mb{U}\sqrt{\bs\Lambda}\mb{W}^\top$. Then by Condition $\ref{cond:invariant}$, we have
$\Omega_\phi(\mb{A}^*) = \Omega_\phi(\tilde{\mb{A}})$. Define 
$$L(\mb{A}) = \E \Big[ \frac{1}{|\mathcal{S}|}\sum_{(\mb{x},\mb{y})\in \mathcal{S}}\|\mb{Ax}-\mb{Ay}\|_2^{2} +\frac{1}{|\mathcal{D}|} \sum_{(\mb{x},\mb{y})\in \mathcal{D}}\textrm{max}(0, \tau-\|\mb{Ax}-\mb{Ay}\|_2^{2}) \Big].$$ Assuming Eq.(\ref{eq:feature space}) does not hold, we prove $L(\mb{A}^*) > L(\tilde{\mb{A}})$, i.e., $\mb{A}^*$ is not the optimal solution. We consider two cases: $\xi = 0$ and $\xi \neq 0$. Define $h(\mb{A}^*, \xi) = L(\mb{A}^*)$ and $h(\tilde{\mb{A}}, \xi) = L(\tilde{\mb{A}})$. When $\xi=0$, we have:
\begin{align*}
h(\mb{A}^*, 0) & = \E \Big[ \frac{1}{|\mathcal{S}|}\sum_{(\mb{x}, \mb{y})\in \mathcal{S}}\|\mb{A}^*\bs\mu_{c(\mb{x})}-\mb{A}^*\bs\mu_{c(\mb{y})}\|_2^{2}+\frac{1}{|\mathcal{D}|} \sum_{(\mb{x}, \mb{y})\in \mathcal{D}} \max (0, \tau-\|\mb{A}^*\bs\mu_{c(\mb{x})}-\mb{A}^*\bs\mu_{c(\mb{y})}\|_2^{2}) \Big] \\
	& = \sum_{j\ne k} p_{jk} \max (0, \tau-\|\mb{A}^*(\bs\mu_j-\bs\mu_k)\|_2^2),
\end{align*}
and 
$$h(\tilde{\mb{A}}, 0) =\sum_{j\ne k} p_{jk} \max (0, \tau-\|\tilde{\mb{A}}(\bs\mu_j-\bs\mu_k)\|_2^2).$$

Since Eq.(\ref{eq:feature space}) does not hold by assumption, there exists $\mb{w}_0 \in \mc{G}$, $\mb{w}_0 \notin \mr{span}(\mb{A}^*)$. Denote $\mc{H} = \{\mb{v} \in \mc{G}: \mb{v} \perp \mb{w}_0\}$ and its $K-2$ nonzero feature values as $\beta'_1, \beta'_2, \cdots \beta'_{K-2}$. $\forall \mb{u} \in \mr{span}(\mb{A}^*)$, let $\mb{u}'$ be the projection of $\mb{u}$ to the space $\mc{H}$ and $\mb{u}'$ is rescaled to have norm 1. Then $\alpha(\mb{u}') \ge \alpha(\mb{u})$. Thus, $\forall r$, the $r$-th feature value of $\mr{span}(\mb{A}^*)$ is no larger than the $r$-th feature value of $\mc{G}$. By Lemma \ref{lem:constrained feature values}, we have $\sum_{l=1}^r \gamma_l \beta'_l \le \sum_{l=1}^r \gamma_l\beta_l$. By the definition of feature values, we have
$$\sum_{j\ne k}p_{jk}\|\mb{A}^*(\bs\mu_j-\bs\mu_k)\|_2^2 = \sum_{l=1}^R \gamma_l \alpha(\mb{a}_l)  \le \sum_{l=1}^R \gamma_l \beta'_l.$$ 
Since $\mc{H}$ has only $K-2$ nonzero feature values, we have
$$\sum_{l=1}^R \gamma_l\beta'_l = \sum_{l=1}^{K-2} \gamma_l \beta'_l \le \sum_{l=1}^{K-2} \gamma_l\beta_l = \sum_{l=1}^{K-1} \gamma_l \alpha(\mb{w}_l)  - \gamma_{K-1}\beta_{K-1} = \sum_{j\ne k}p_{jk}\|\tilde{\mb{A}}(\bs\mu_j-\bs\mu_k)\|_2^2 - \gamma_{K-1}\beta_{K-1}.$$
So we have
$$ \sum_{j\ne k}p_{jk}\|\tilde{\mb{A}}(\bs\mu_j-\bs\mu_k)\|_2^2 \ge \sum_{j\ne k}p_{jk}\|\mb{A}^*(\bs\mu_j-\bs\mu_k)\|_2^2 + \gamma_{K-1}\beta_{K-1}.$$
Next, we establish a relationship between $h(\mb{A}^*, 0)$ and $h(\tilde{\mb{A}},0)$, which is given in the following lemma.

\begin{lemma}\label{lem:lambda thre}
There exist constants $\tau_0, \gamma_0$ which are determined by $p_1, p_2, \cdots p_K$ and $\bs\mu_1, \bs\mu_2, \cdots,\bs\mu_K$, such that if $\tau \ge \tau_0, \gamma \ge \gamma_0$, then we have
$$h(\mb{A}^*, 0) - h(\tilde{\mb{A}},0) > \frac{1}{2}\gamma_{K-1}\beta_{K-1}.$$
\end{lemma}
\begin{proof}
If $\|\tilde{\mb{A}}(\bs\mu_j-\bs\mu_k)\|_2^2 \le \tau$ and $\|\mb{A}^*(\bs\mu_j-\bs\mu_k)\|_2^2 \le \tau$ for all $j\ne k$, we have $h(\mb{A}^*, 0) - h(\tilde{\mb{A}},0) =\gamma_{K-1}\beta_{K-1}.$ Since $\max_{j\ne k} \|\bs\mu_j - \bs\mu_k\|_2 = B_0$, we have 
\begin{equation}\label{eq:mu up}\begin{aligned}
\|\mb{A}^*(\bs\mu_j - \bs\mu_k)\|_2^2 \le \tr(\bs\Lambda)B_0^2, &\\
\|\tilde{\mb{A}}(\bs\mu_j - \bs\mu_k)\|_2^2 \le \tr(\bs\Lambda)B_0^2, &\ \  \forall j \ne k.
\end{aligned}\end{equation}
Select $\tau_0$ such that $\tau_0 \ge K(1 + \epsilon_0)B_0^2$, where $\epsilon_0$ is any positive constant. For the VND and LDD regularizers, as $\gamma \to \infty$, $\bs\Lambda \to \mb{I}_R$. Thereby, there exists $\gamma_0$, such that if $\gamma \ge \gamma_0$, $\forall j, \ |\lambda_j - 1|\le \epsilon$. Hence, if $\gamma \ge \gamma_0, \tau \ge \tau_0$, $$\tr(\bs\Lambda)B_0^2 \le K(1+\epsilon_0)B_0^2 \le \tau_0.$$
Combining this inequality with Eq.(\ref{eq:mu up}), we finish the proof.
\end{proof}

Now we continue to prove Lemma \ref{thm:recover feature space}. In Lemma \ref{lem:lambda thre}, we have already proved that $h(\mb{A}^*, 0)$ is strictly larger than $h(\tilde{\mb{A}},0)$. We then prove that if the noise is smaller than a certain value, $h(\mb{A}^*, \xi)$ is strictly larger than $h(\tilde{\mb{A}},\xi)$. By the definition of $\xi$, we have
\begin{align}\label{eq:kappa}
	& |h(\mb{A}^*,\xi) - h(\mb{A}^*,0)| \nn \\
	\le & \E \frac{1}{|\mc{S}|} \sum_{(x,y)\in\mc{S}} \|\mb{A}^*(\mb{x}-\mb{y})\|_2^2 +\E \frac{1}{|\mc{D}|} \sum_{(x,y)\in\mc{D}} [\|\mb{A}^*(\mb{x}-\mb{y})\|_2^2 - \|\mb{A}^*(\bs\mu_{c(\mb{x})}-\bs\mu_{c(\mb{y})})\|_2^2] \nn \\
	\le & 4\tr(\bs\Lambda)\xi^2 + (4B_0\xi + 4\xi^2)\tr(\bs\Lambda) \nn \\
	= & 8\xi^2\tr(\bs\Lambda) + 4B_0\xi\tr(\bs\Lambda) . 
\end{align}
Similarly, we have
\begin{equation}\label{eq:*kappa}
|h(\mb{A}^*,\xi) - h(\mb{A}^*,0)| \le 8\xi^2\tr(\bs\Lambda) + 4B_0\xi\tr(\bs\Lambda).
\end{equation}
Combining Lemma \ref{lem:lambda thre} with Eq.(\ref{eq:kappa}) and Eq.(\ref{eq:*kappa}), we have if $\xi \le \frac{-B_0 + \sqrt{B_0^2 + \gamma_{K-1}\beta_{K-1} / (2\tr(\bs\Lambda))}}{4}$, then $L(\mb{A}^*) > L(\tilde{\mb{A}})$, i.e., $\mb{A}^*$ is not the global optimal solution. By contradiction, Eq.(\ref{eq:feature space}) holds. The proof completes.
\end{proof}

\subsection{Proof of Lemma \ref{thm:inb}}
\begin{proof}
For any vector $\mb{u} \in \mc{G}$, since the condition of Lemma \ref{thm:recover feature space} is satisfied, we have $\mb{u} \in \spann(\mb{A}^*)$. Recall $\mb{A}^{*\top}\mb{A}^* = \mb{V\Gamma V}^\top$ and $\mb{V} = [\mb{v}_1, \mb{v}_2, \cdots \mb{v}_R]$. We can denote $\mb{u}$ as $\mb{u} = \| \mb{u} \| \sum_{j=1}^R t_j \mb{v}_j$, where $\sum_{j=1}^R t_j^2 = 1$. Then we have $\forall \mb{u} \in \mc{G}$, 
$$\mb{u}^\top \mb{A}^{*\top}\mb{A}^* \mb{u} = \sum_{j=1}^R \langle \mb{v}_j, \mb{u} \rangle^2 \lambda_j = \sum_{j=1}^R \| \mb{u} \|^2 t_j^2 \lambda_j \le \| \mb{u} \|^2 \lambda_1.$$
Similarly, we have $\mb{u}^\top \mb{A}^{*\top}\mb{A}^* \mb{u} \ge \| \mb{u} \|^2 \lambda_R$. 
Noting $\forall j \ne k$, $\bs\mu_j - \bs\mu_k \in \mc{G}$, we have 
$$\eta \le \cond(\mb{A}^*\mb{A}^{*\top}) \frac{\max_{j \ne k} \lVert \bs\mu_j - \bs\mu_k \rVert^2}{\min_{j\ne k} \lVert \bs\mu_j - \bs\mu_k \rVert^2}.$$ 
Combining this inequality with $\cond(\mb{A}^*\mb{A}^{*\top}) \le g(\Omega_{\phi}(\mb{A}^*))$, we complete the proof.
\end{proof}

\subsection{Proof of Lemma \ref{thm:cond regu}}

\begin{proof}
We first prove the result about the VND regularizer. Define scalar function $s(x) = x \log x - x + 1$ and denote $\cond(\mb{A}^*\mb{A}^{*\top}) = c$. Since $s'(x) = \log x$, and $s(1) = 0$, we have
\begin{align*}
\Omega_{vnd}(\mb{A}^*) & = \sum_{j=1}^R s(\lambda_j) \\
	& \ge s(\lambda_1) + s(\lambda_R)\\
	& = \lambda_1 \log \lambda_1 - \lambda_1 + \frac{\lambda_1}{c}\log \frac{\lambda_1}{c} - \frac{\lambda_1}{c} + 2
\end{align*}
Define $F(\lambda_1, c) = \lambda_1 \log \lambda_1 - \lambda_1 + \frac{\lambda_1}{c}\log \frac{\lambda_1}{c} - \frac{\lambda_1}{c} + 2$. We aim to maximize $c$, so 
$$\frac{\partial}{\partial \lambda_1}F(\lambda_1, c) = 0.$$
This equation has a unique solution: $\log \lambda_1 = \frac{\log c}{c + 1}$. Therefore we have
$$c^{1/(c+1)}(1+\frac{1}{c}) \ge 2 - \Omega_{vnd}(\mb{A}^*).$$
Define $f(c) = c^{1/(c+1)}(1+\frac{1}{c})$. Its derivative is: $f'(c) = -\frac{\log c}{c(c+1)}c^{1/(c+1)}$.
Analyzing $f'(c)$, we know that $f(c)$ increases on $(0, 1]$, decreases on $[1, \infty)$, and $f(1)=2$. Also we have the following limits:
$$\lim_{c \to 0} f(c) = 0, \ \lim_{c \to \infty} f(c) = 1.$$
We denote the inverse function of $f(\cdot)$ on $[1,\infty)$ as $f^{-1}(\cdot)$. Then for any $\Omega_{vnd}(\mb{A}^*) < 1$, we have 
$$ \cond(\mb{A}^*\mb{A}^{*\top}) \le f^{-1}(2-\Omega_{vnd}(\mb{A}^*)).$$ 

Next we prove the result for the LDD regularizer $\Omega_{ldd}(\mb{A}^*)$. Define scalar function $s(x) = x  - \log x - 1$ and denote $\cond(\mb{A}^*\mb{A}^{*\top}) = c$. Since $s'(x) = 1 - \frac{1}{x}$ and $s(1) = 0$, we have
\begin{align*}
\Omega_{ldd}(\mb{A}^*) & = \sum_{j=1}^R s(\lambda_j) \\
	& \ge s(\lambda_1) + s(\lambda_R)\\
	& = \lambda_1  - \log \lambda_1 + \frac{\lambda_1}{c} - \log \frac{\lambda_1}{c} - 2
\end{align*}
Therefore we have
\begin{align*}
\log c & \le \Omega_{ldd}(\mb{A}^*) + 2\log \lambda_1 - \lambda_1 (1+ \frac{1}{c}) + 2 \\
	& \le \Omega_{ldd}(\mb{A}^*) + 2\log \lambda_1 - \lambda_1 + 2 \\
	& \le \Omega_{ldd}(\mb{A}^*) + 2\log 2 - 2 + 2 \\
	& = \Omega_{ldd}(\mb{A}^*) + 2 \log 2
\end{align*}
The third inequality is obtained from the following fact: the scalar function $\log x - x$ gets its maximum when $x = 2$. Further, we have 
$$c \le 4 e^{\Omega_{ldd}(\mb{A}^*)}.$$
The proof completes.

\end{proof}

\section{Proof of Theorem \ref{thm:gen vnd}}
\subsection{Proof Sketch}
Part of the proof is tailored to the CVND regularizer. Extensions to CSFN and CLDD are given later. The proof is based on Rademacher complexity (RC)~\cite{bartlett2003rademacher}, which measures the complexity of a hypothesis class. In MDML, the Rademacher complexity $\mc{R}(\mc{M})$ of the function class $\mc{M}$ is defined as:

$$\mc{R}(\mc{M}) = \E_{\mc{S}, \mc{D}, \sigma}  \sup_{\mb{M} \in \mathcal{M}} \frac{1}{m} \sum_{i=1}^m \sigma_i (\mb{x}_i-\mb{y}_i)^\top \mb{M}(\mb{x}_i-\mb{y}_i) $$ 
where $m$ is the number of data pairs in the training data ($m=|\mc{S}|+|\mc{D}|$), $\sigma_i\in\{-1,1\}$ is the Rademacher variable and $\sigma=(\sigma_1, \sigma_2, \cdots \sigma_m)$.

We first establish a upper bound of the generalization error based on RC. Intuitively, a less-complicated hypothesis class generalizes better on unseen data. Then we upper bound the RC based on the CBMD regularizers. Combining the two steps together, we establish upper bounds of the generalization error based on CBMD regularizers. 

The following lemma presents the RC-based upper bound of the generalization error. Its proof is adapted from \cite{bartlett2003rademacher}.  

\begin{lemma}\label{lem:bound lip} 
With probability at least $1-\delta$, we have
\begin{equation}\label{eq:bound lip}
\begin{aligned}
&\sup_{\mb{M} \in \mc{M}} (L(\mb{M}) - \hat{L}(\mb{M})) \le 2\mc{R}(\mc{M}) +  \max(\tau,\sup_{\scriptsize\begin{array}{c}
(\mb{x}, \mb{y}) \in \mc{S}\\
\mb{M} \in \mc{M}
\end{array}}
 (\mb{x}-\mb{y})^\top \mb{M} (\mb{x}-\mb{y})) \sqrt{\frac{2\log(1/\delta)}{m}}.
\end{aligned}
\end{equation}
\end{lemma}
For the second term in the bound, it is easy to verify
\begin{equation}\label{eq:c}
\sup_{\scriptsize\begin{array}{c}
(\mb{x}, \mb{y}) \in \mc{S}\\
\mb{M} \in \mc{M}
\end{array}}
(x-y)^\top \mb{M} (\mb{x}-\mb{y}) \le \sup_{\mb{M} \in \mc{M}} \tr(\mb{M}) \sup_{(\mb{x}, \mb{y}) \in \mc{S}}\|\mb{x}-\mb{y}\|_2^2.\end{equation}
Now we focus on the first term. We denote $\mb{z} = \mb{x}-\mb{y}$, $\mb{z}_i = \mb{x}_i - \mb{y}_i$. 

\begin{lemma}\label{lem:rad trace}
Suppose $\sup_{\|\mb{v}\|_2\le 1, \mb{z}} |\mb{v}^\top \mb{z}| \le B$, then we have
\begin{equation}\label{eq:rad trace}
\mc{R}(\mc{M}) \le \frac{2B^2}{\sqrt{m}} \sup_{\mb{M} \in \mc{M}}\tr(\mb{M}).
\end{equation}
\end{lemma}

We next show that $\tr(\mb{M})$ can be bounded by the CVND regularizer $\hat{\Omega}_{vnd}(\mb{M})$.

\begin{lemma}\label{lem:tr vn}
For the convex VND regularizer $\hat{\Omega}_{vnd}(\mb{M})$, for any positive semidefinite matrix $\mb{M}$, we have
$$\tr(\mb{M}) \le \hat{\Omega}_{vnd}(\mb{M}).$$
\end{lemma}

Combining Lemma \ref{lem:bound lip}, \ref{lem:rad trace}, \ref{lem:tr vn} and Eq.(\ref{eq:c}) and noting that $\mc{E}=L(\hat{\mb{M}}^*)-\hat{L}(\hat{\mb{M}}^*)\leq \sup_{\mb{M} \in \mc{M}} (L(\mb{M}) - \hat{L}(\mb{M}))$ and $\hat{\Omega}_{vnd}(\mb{M})\leq C$ ($C$ is the upper bound in the hypothesis class $\mc{M}$), we complete the proof of the first bound in Theorem \ref{thm:inb}. 

In the sequel, we present detailed proofs of these lemmas and the extension to CSNF and CLDD.

\subsection{Proof of Lemma \ref{lem:rad trace}}\label{app:proof gen}
\begin{proof}
For any $\mb{M} \in \mc{M}$, denote its spectral decomposition as $\mb{M} = \mb{V}\mb{\Pi} \mb{V}^\top$, where $\mb{V}$ is standard orthogonal matrix and $\mb{\Pi}$ is diagonal matrix. Denote $\mb{V} = (\mb{v}_1, \mb{v}_2, \cdots \mb{v}_D)$, $\mb{\Pi} = \mathrm{diag}(\pi_1, \pi_2, \cdots \pi_D)$, then we have
\begin{align*}
\mc{R}(\mc{M})& = \E_{\mc{S},\mc{D}, \sigma} \sup_{\mb{M}\in \mathcal{M}} \left[ \frac{1}{m} \sum_{i=1}^m \sigma_i \mb{z}_i^T \mb{M} \mb{z}_i \right] \\
	& = \frac{1}{m} \E_{\mc{S},\mc{D}, \sigma} \sup_{\mb{M}\in \mathcal{M}}\left[  \sum_{i=1}^m \sigma_i \sum_{j=1}^D \pi_j (\mb{v}_j^\top \mb{z}_i)^2 \right] \\
	& = \frac{1}{m} \E_{\mc{S},\mc{D}, \sigma}  \sup_{\mb{M}\in \mathcal{M}} \left[ \sum_{j=1}^D \pi_j \sum_{i=1}^m \sigma_i (\mb{v}_j^\top \mb{z}_i)^2 \right] \\
	& = \frac{1}{m} \E_{\mc{S},\mc{D}, \sigma}  \sup_{\mb{M}\in \mathcal{M}} \left[ \sum_{j=1}^D \pi_j \sup_{\|\mb{v}\|_2 \le 1} \sum_{i=1}^m \sigma_i (\mb{v}^\top \mb{z}_i)^2 \right] \\
	& = \frac{1}{m} \E_{\mc{S},\mc{D}, \sigma}  \sup_{\mb{\Pi}}  \sum_{j=1}^D \pi_j  \sup_{\|\mb{v}\|_2 \le 1} \sum_{i=1}^m \sigma_i (\mb{v}^\top \mb{z}_i)^2  \\
	& \le \frac{1}{m} \sup_{\mb{M} \in \mc{M}}\tr(\mb{M}) \E_{\mc{S},\mc{D}, \sigma}  \sup_{\|\mb{v}\|_2 \le 1} \sum_{i=1}^m \sigma_i (\mb{v}^\top \mb{z}_i)^2.
\end{align*}
Since $(\mb{v}^\top \mb{z})^2$ is Lipschitz continuous w.r.t $\mb{v}^\top \mb{z}$ with constant $2\sup_{\|\mb{v}\|_2\le 1, \mb{z}}\mb{v}^\top \mb{z}$, according to the composition property~\cite{bartlett2003rademacher} of Rademacher complexity on Lipschitz continuous functions, we have
\begin{align*}
\mc{R}(\mc{M}) & \le \frac{1}{m} 2\sup_{\|\mb{v}\|_2\le 1, \mb{z}}(\mb{v}^\top \mb{z}) \ \sup_{\mb{M} \in \mc{M}}\tr(\mb{M})  \E_{\mc{S},\mc{D}, \sigma}  \sup_{\|\mb{v}\|_2 \le 1} \sum_{i=1}^m \sigma_i \mb{v}^\top \mb{z}_i \\
	& = 2\frac{B}{m}\sup_{\mb{M} \in \mc{M}}\tr(\mb{M}) \E_{\mc{S},\mc{D}, \sigma}  \sup_{\|\mb{v}\|_2 \le 1} \sum_{i=1}^m \sigma_i \mb{v}^\top \mb{z}_i \\
	& \le 2\frac{B}{m}\sup_{\mb{M} \in \mc{M}}\tr(\mb{M}) \E_{\mc{S},\mc{D}, \sigma} \sup_{\|\mb{v}\|_2 \le 1} \|\mb{v}\|_2 \|\sum_{i=1}^m \sigma_i \mb{z}_i\|_2 \\
	& = 2\frac{B}{m}\sup_{\mb{M} \in \mc{M}}\tr(\mb{M}) \E_{\mc{S},\mc{D}, \sigma} \sqrt{(\sum_{i=1}^m \sigma_i \mb{z}_i)^2}.
\end{align*}
By Jensen's inequality, we have
\begin{align*}
\mc{R}(\mc{M}) & \le 2\frac{B}{m}\sup_{\mb{M} \in \mc{M}}\tr(\mb{M}) \E_{\mc{S},\mc{D}} \sqrt{\E_{\sigma}(\sum_{i=1}^m \sigma_i \mb{z}_i)^2} \\
	& = \le 2\frac{B}{m}\sup_{\mb{M} \in \mc{M}}\tr(\mb{M}) \E_{\mc{S},\mc{D}} \sqrt{\sum_{i=1}^m \mb{z}_i^2}\\
	& \le \frac{2B^2}{\sqrt{m}} \sup_{\mb{M} \in \mc{M}}\tr(\mb{M}).
\end{align*}
\end{proof}

\subsection{Proof of lemma \ref{lem:tr vn}}
\begin{proof}
For any positive semidefinite matrix $\mb{M}$, we use notations $\mb{V}, \mb{\Pi}, \pi_j, 1\le j \le D$ as they are defined in Section~\ref{app:proof gen}. By the definition of the convex VND regularizer, we have
\begin{align*}
\hat{\Omega}_{vnd}(\mb{M}) = & \Gamma_{vnd}(\mb{M}+\epsilon \mb{I}_D, \mb{I}_D) + \tr(\mb{M}) \\
= & \tr[(\mb{M}+\epsilon \mb{I}_D)\log(\mb{M}+\epsilon \mb{I}_D) - (\mb{M} + \epsilon \mb{I}_D)\log \mb{I}_D - (\mb{M}+\epsilon) +\mb{I}_D]+\tr(\mb{M}) \\
= & \sum_{j=1}^D [(\pi_j + \epsilon) \log (\pi_j + \epsilon) - (\pi_j + \epsilon) + 1] + \sum_{j=1}^D \pi_j \\
= & \sum_{j=1}^D [(\lambda_j + \epsilon) \log (\lambda_j + \epsilon) - \epsilon + 1]
\end{align*}
Denote $\bar{\pi} = (\sum_{j=1}^D \pi_j) / D = \tr(\mb{M}) / D$, then by Jensen's inequality, we have
$$\sum_{j=1}^D (\lambda_j + \epsilon) \log (\lambda_j + \epsilon) \ge D (\bar{\pi} + \epsilon) \log (\bar{\pi} + \epsilon).$$
Since $\forall x \in \mathbb{R}_+, x - 1 \le x \log x$, so we have
\begin{align*}
\bar{\pi} + \epsilon - 1 & \le (\bar{\pi} + \epsilon) \log (\bar{\pi} + \epsilon)\\
	& \le \frac{1}{D} \sum_{j=1}^D (\lambda_j + \epsilon) \log (\lambda_j + \epsilon)\\
	& \le \frac{1}{D} \hat{\Omega}_{vnd}(\mb{M}) + \epsilon - 1.
\end{align*}
Therefore we have
$$\tr(\mb{M}) \le \hat{\Omega}_{vnd}(\mb{M}).$$
\end{proof}

\subsection{Generalization error bound for the convex SFN regularizer}
In this section we prove generalization error bounds for the convex SFN regularizer. The CSFN is composed of two parts. One is the squared Frobenius norm of $\mb{M} - \mb{I}_D$ and the other is the trace of $\mb{M}$. We have already established a relationship between $\tr(\mb{M})$ and $\mc{R}(\mc{M})$. Now we analyze the relationship between $\|\mb{M}-\mb{I}_D\|_F$ and $\mc{R}(\mc{M})$, which is given in the following lemma.
\begin{lemma}\label{lem:rad fro}
Suppose $\sup_{\|\mb{v}\|_2\le 1, \mb{z}} |\mb{v}^\top \mb{z}| \le B$, then we have
\begin{equation}\label{eq:rad fro}
\mc{R}(\mc{M}) \le \frac{B^2}{\sqrt{m}} \sup_{\mb{M} \in \mc{M}} \|\mb{M}-\mb{I}_D\|_F
\end{equation}
\end{lemma}

\begin{proof}
Denote $M(j, k) = a_{jk}$, and $\delta_{jk} = \mathrm{I}_{\{j=k\}}$, $\mb{z}_i = (z_{i1}, z_{i2}, \cdots z_{id})$, then we have
\begin{align*}
\mc{R}(\mc{M}) = & \frac{1}{m} \E_{\mc{S}, \mc{D}, \sigma} \sup_{\mb{M} \in \mathcal{M}}\left[ \sum_{j, k} a_{jk} \sum_{i=1}^m \sigma_i z_{ij}z_{ik} \right] \\
	= & \frac{1}{m}  \E_{\mc{S}, \mc{D}, \sigma} \sup_{\mb{M} \in \mathcal{M}}\Bigg[ \sum_{j, k} (a_{jk}-\delta_{jk}) \sum_{i=1}^m \sigma_i z_{ij}z_{ik} + \sum_{j, k}\delta_{jk} \sum_{i=1}^m \sigma_i z_{ij}z_{ik} \Bigg] \\
	\le & \frac{1}{m}  \E_{\mc{S}, \mc{D}, \sigma} \sup_{\mb{M} \in \mathcal{M}}\left[ \|\mb{M}-\mb{I}_D\|_F \sqrt{\sum_{j, k} (\sum_{i=1}^m \sigma_iz_{ij}z_{ik})^2} \right]  
\end{align*}
Here the inequality is attained by Cauchy's inequality. Applying Jensen's inequality, we have
\begin{align*}
\mc{R}(\mc{M}) & \le \frac{1}{m} \sup_{\mb{M} \in \mathcal{M}} \|\mb{M}-\mb{I}_D\|_F \ \E_{\mc{S}, \mc{D}} \left[ \sqrt{\E_{\sigma} \sum_{j, k} (\sum_{i=1}^m \sigma_i z_{ij}z_{ik})^2} \right] \\
	& = \frac{1}{\sqrt{m}} \sup_{\mb{M} \in \mathcal{M}} \|\mb{M}-\mb{I}_D\|_F \ \E_{\mc{S}, \mc{D}} \left[ \sqrt{ \sum_{j, k} z_{ij}^2 z_{ik}^2} \right]
\end{align*}
Recalling the definition of $B$, we have
$$\mc{R}(\mc{M})  \le \frac{B^2}{\sqrt{m}} \sup_{\mb{M} \in \mathcal{M}} \|\mb{M}-\mb{I}_D\|_F.$$
\end{proof}

We now bound the generalization error with the convex SFN regularizer, which is given in the following lemma.
\begin{lemma}
\label{thm:gen sfn}
Suppose $\sup_{\|\mb{v}\|_2\le 1, \mb{z}} |\mb{v}^\top \mb{z}| \le B$, then with probability at least $1-\delta$, we have
\begin{align*}
  \sup_{\mb{M} \in \mc{M}} (L(\mb{M}) - \hat{L}(\mb{M})) \leq  \frac{2B^2}{\sqrt{m}}\min(2\hat{\Omega}_{sfn}(\mb{M}), \sqrt{\hat{\Omega}_{sfn}(\mb{M})})+ \max(\tau, \hat{\Omega}_{sfn}(\mb{M})) \sqrt{\frac{2\log(1/\delta)}{m}}.
\end{align*}
\end{lemma}

\begin{proof}
For the convex SFN regularizer $\hat{\Omega}_{sfn}(\mb{M})$, we have $\tr(\mb{M}) \le \hat{\Omega}_{sfn}(\mb{M})$ and $\|\mb{M}-\mb{I}_D\| \le \hat{\Omega}_{sfn}(\mb{M})$. By Eq.(\ref{eq:c}), we have
\begin{equation}\label{eq:sfn 2}
\sup_{\scriptsize\begin{array}{c}
(\mb{x}, \mb{y}) \in \mc{S}\\
\mb{M} \in \mc{M}
\end{array}}
(x-y)^\top \mb{M} (\mb{x}-\mb{y}) \le \sup_{\mb{M} \in \mc{M}}\hat{\Omega}_{sfn}(\mb{M}) B^2.
\end{equation}
By Lemma \ref{lem:rad trace} and \ref{lem:rad fro}, we have
\begin{equation}\label{eq:sfn 1}
\mc{R}(\mc{M}) \le \frac{B^2}{\sqrt{m}} \min(2\hat{\Omega}_{sfn}(\mb{M}), \sqrt{\hat{\Omega}_{sfn}(\mb{M})}).
\end{equation}
Substituting Eq.(\ref{eq:sfn 1}) and Eq.(\ref{eq:sfn 2}) into Lemma \ref{lem:bound lip}, we have
\begin{align*}
  \sup_{\mb{M} \in \mc{M}} (L(\mb{M}) - \hat{L}(\mb{M})) \leq  \frac{2B^2}{\sqrt{m}}\min(2\hat{\Omega}_{sfn}(\mb{M}), \sqrt{\hat{\Omega}_{sfn}(\mb{M})})+ \max(\tau, \hat{\Omega}_{sfn}(\mb{M})) \sqrt{\frac{2\log(1/\delta)}{m}}.
\end{align*}
\end{proof}

Noting that $\mc{E}=L(\hat{\mb{M}}^*)-\hat{L}(\hat{\mb{M}}^*)\leq \sup_{\mb{M} \in \mc{M}} (L(\mb{M}) - \hat{L}(\mb{M}))$ and $\hat{\Omega}_{sfn}(\mb{M})\leq C$, we conclude 
$ \mc{E}\leq  \frac{2B^2}{\sqrt{m}}\min(2C, \sqrt{C})+ \max(\tau, C) \sqrt{\frac{2\log(1/\delta)}{m}}$.

\subsection{Generalization error bound for the convex LDD regularizer}
Starting from Lemma \ref{lem:bound lip}, we bound $\mc{R}(\mc{M})$ and $\sup_{\mb{M} \in \mc{M}} \tr(\mb{M})$ which are given in the following two lemmas.

\begin{lemma}\label{lem:ldd rad}
Suppose $\sup_{\|\mb{v}\|_2\le 1, \mb{z}} |\mb{v}^\top \mb{z}| \le B$, then we have
$$\mc{R}(\mc{M}) \le \frac{B}{\sqrt{m}} \frac{\hat{\Omega}_{ldd}(\mb{M})}{\log(1/\epsilon)-1}.$$
\end{lemma}

\begin{proof}
We first perform some calculation on the convex LDD regularizer. 
\begin{equation}
\label{eq:cldd_reformu}
\begin{array}{lll}
\hat{\Omega}_{ldd}(\mb{M})& = & \Gamma_{ldd}(\mb{M} + \epsilon \mb{I}_D, \mb{I}_D) - (1 + \log \epsilon) \tr(\mb{M}) \\
&	= & \tr( (\mb{M} + \epsilon \mb{I}_D) \mb{I}_D^{-1}) - \log \det ((\mb{M} + \epsilon \mb{I}_D) \mb{I}_D^{-1}) - D - (1+ \log \epsilon)\tr(\mb{M}) \\
&	= & \sum_{j=1}^D (\pi_j+\epsilon) - \sum_{j=1}^D \log (\pi_j + \epsilon) - D - (1 + \log \epsilon) \sum_{j=1}^D \pi_j \\
&	= & \log(\frac{1}{\epsilon})\sum_{j=1}^D \pi_j - \sum_{j=1}^D \log (\pi_j + \epsilon) - D(1 - \epsilon).
\end{array}
\end{equation}

Now we upper bound the Rademacher complexity using the CLDD regularizer.
\begin{align*}
\log(\frac{1}{\epsilon})\mc{R}(\mc{M}) = & \frac{\log(\frac{1}{\epsilon})}{m} \E_{\mc{S},\mc{D}, \sigma}  \sup_{\mb{M} \in \mc{M}} \left[ \sum_{j=1}^D \pi_j \sum_{i=1}^m \sigma_i (\mb{v}_j^\top \mb{z}_i)^2 \right] \\
	\le & \frac{1}{m}  \E_{\mc{S},\mc{D}, \sigma} \sup_{\mb{\Pi}} \sum_{j=1}^D [(\log(\frac{1}{\epsilon})\pi_j - \log(\pi_j + \epsilon)) +  \log(\pi_j + \epsilon)] \sup_{\|\mb{v}\|_2 \le 1} \sum_{i=1}^m \sigma_i (\mb{v}^\top \mb{z}_i)^2 
\end{align*}
Similar to the proof of Lemma \ref{lem:rad trace}, we have
\begin{equation}
\label{eq:rc_bound}
\begin{array}{lll}
\log(\frac{1}{\epsilon})\mc{R}(\mc{M}) & \le& \frac{2B^2}{\sqrt{m}} \sup_{\mb{\Pi}} \sum_{j=1}^D [(\log(\frac{1}{\epsilon})\pi_j - \log(\pi_j + \epsilon)) +  \log(\pi_j + \epsilon)] \\
	& \le &\frac{2B^2}{\sqrt{m}}  [\sup_{\mb{M} \in \mc{M}} \hat{\Omega}_{ldd}(\mb{M}) +  \sup_{\mb{M} \in \mc{M}} \sum_{j=1}^D \log(\pi_j + \epsilon)].
\end{array}
\end{equation}
Denoting $A =  \sum_{j=1}^D \log(\pi_j + \epsilon)$, we bound $A$ with $\hat{\Omega}_{ldd}(\mb{M})$. Denoting $\bar{\pi} = (\sum_{j=1}^D \pi_j) / D = \tr(\mb{M}) / D$, by Jensen's inequality, we have
\begin{equation}\label{eq:A pi}
A \le D \log (\bar{\pi} + \epsilon), 
\end{equation}
then $\bar{\pi} \ge e^{A/D} - \epsilon$. Replacing $\bar{\pi}$ with $A$ in Eq.(\ref{eq:cldd_reformu}), we have
\begin{align*}
\hat{\Omega}_{ldd}(\mb{M}) \ge & D \log(1/\epsilon)(e^{A/D} - \epsilon) - A - D(1 - \epsilon)\\
	\ge & D \log(1/\epsilon)(\frac{A}{D} + 1 - \epsilon) - A - D(1 - \epsilon)\\
	= & (\log(1/\epsilon) - 1) A + [\log(\frac{1}{\epsilon}) - 1] D (1 - \epsilon).
\end{align*}
Further, 
\begin{equation}\label{eq:A up}
A \le \frac{\hat{\Omega}_{ldd}(\mb{M})}{\log(\frac{1}{\epsilon}) - 1} - D(1 - \epsilon).
\end{equation}
Substituting this upper bound of $A$ into Eq.(\ref{eq:rc_bound}),
we have
$$\mc{R}(\mc{M}) \le \frac{2B^2}{\sqrt{m}} \frac{\sup_{\mb{M} \in \mc{M}}\hat{\Omega}_{ldd}(\mb{M})}{\log(1/\epsilon)-1}. 
$$
\end{proof}

The next lemma shows the bound of $\tr(\mb{M})$.
\begin{lemma}\label{lem:ldd tr}
For any positive semidefinite matrix $\mb{M}$, we have
$$\tr(\mb{M}) \le \frac{\hat{\Omega}_{ldd}(\mb{M}) - D \epsilon}{\log(\frac{1}{\epsilon}) - 1}.$$
\end{lemma}

\begin{proof}
\begin{align*}
\hat{\Omega}_{ldd}(\mb{M}) \ge & D \log(1/\epsilon)\bar{\pi} - D\log(\bar{\pi} + \epsilon) - D(1 - \epsilon)\\
	\ge & D \log(1/\epsilon) \bar{\pi} + D (1 - \bar{\pi}) - D(1 -\epsilon)\\
	= & D[\log(1/\epsilon) - 1] \bar{\pi} + D \epsilon.
\end{align*}
Then
$$\tr(\mb{M}) = D \bar{\pi} \le \frac{\hat{\Omega}_{ldd}(\mb{M}) - D \epsilon}{\log(\frac{1}{\epsilon}) - 1}.$$
\end{proof}

Combining Lemma \ref{lem:ldd rad}, \ref{lem:ldd tr}, and \ref{lem:bound lip}, we get the following generalization error bound w.r.t the convex LDD regularizer.
\begin{lemma}
Suppose $\sup_{\|\mb{v}\|_2\le 1, \mb{z}} |\mb{v}^\top \mb{z}| \le B$, then with probability at least $1-\delta$, we have
\begin{align*}
 \sup_{\mb{M} \in \mc{M}} (L(\mb{M}) - \hat{L}(\mb{M})) 
\le  \frac{4B^2\ \hat{\Omega}_{ldd}(\mb{M})}{\sqrt{m}\ [\log(1/\epsilon) - 1]} + \max\Big(\tau,  \frac{\hat{\Omega}_{ldd}(\mb{M}) - D \epsilon}{\log(\frac{1}{\epsilon}) - 1}\Big) \sqrt{\frac{2\log(1/\delta)}{m}}.
\end{align*}
\end{lemma}

\section{Experiments}

\subsection{Details of Datasets and Feature Extraction}
\label{sec:data}

\paragraph{MIMIC-III}
We extract 7207-dimensional features: (1) 2 dimensions from demographics, including age and gender; (2) 5300 dimensions from clinical notes, including 5000-dimensional bag-of-words (weighted using \ti{tf-idf}) and 300-dimensional Word2Vec \cite{mikolov2013distributed}; (3) 1905-dimensions from lab tests where the zero-order, first-order and second-order temporal features are extracted for each of the 635 lab items. In the extraction of bag-of-words features from clinical notes, we remove stop words, then count the document frequency (DF) of the remaining words. Then we select the largest 5000 words to form the dictionary. Based on this dictionary, we extract \ti{tfidf} features. In the extraction of word2vec features, we train 300-dimensional embedding vector for each word using an open source word2vec tool\footnote{https://code.google.com/archive/p/word2vec/}. To represent a clinical note, we average the embeddings of all words in this note. In lab tests, there are 635 test items in total. An item is tested at different time points for each admission. For an item, we extract three types of temporal features: (1) \ti{zero-order}: averaging the values of this item measured at different time points; (2) \ti{first-order}: taking the difference of values at every two consecutive time points $t$ and $t-1$, and averaging these differences; (3) \ti{second-order}: for the sequence of first-order differences generated in (2), taking the difference (called second-order difference) of values at every two consecutive time points $t$ and $t-1$, and averaging these second-order differences. If an item is missing in an admission, we set the zero-order, first-order and second-order feature values to 0.

\paragraph{EICU}
There are 474 lab test items and 48 vital sign items. Each admission has a past medical history, which is a collection of diseases. There are 2644 unique past diseases. We extract the following features: (1) age and gender; (2) zero, first and second order temporal features of lab test and vital signs; (3) past medical history: we use a binary vector to encode them; if an element in the vector is 1, then the patient had the corresponding disease in the past.

\paragraph{Statistics of imbalanced datasets} Table \ref{tb:im} shows the statistics of imbalanced datasets. 

\begin{table}
\small
\centering
\begin{tabular}{l|c|c|c}
\hline
&MIMIC&EICU&Reuters\\
\hline
Num. of all classes &2833&2175&49\\
Num. of frequent classes &150&120&2\\
Num. of infrequent classes &2683&2055&47\\
Percentage of infrequent classes &94.7\%&94.5\%&95.9\%\\
\hline
Num. of all examples &58K&92K&5931\\
Num. of frequent examples &32712&59064&4108\\
Num. of infrequent examples &25288&32936&1823\\
Percentage of infrequent examples &43.6\%&35.8\%&30.7\%\\
\hline
\end{tabular}
\caption{ Statistics of Imbalanced Datasets.}
\label{tb:im}
\end{table}

\subsection{Additional Experimental Settings}
\label{sec:appen-set}

\begin{table}
\small
\centering
\begin{tabular}{l|c|c|c|c|c|c|c}
\hline
&MIMIC  &EICU  &Reuters&News&Cars&Birds&Act  \\
\cline{1-8}
MDML-$\ell_2$   & 0.01 & 0.01 & 0.001 &  0.1 & 0.01  &0.001 & 0.1\\
MDML-$\ell_1$~\cite{qi2009efficient}   & 0.01 & 0.01 & 0.1&0.01   & 0.01 &0.01 & 0.001 \\
MDML-$\ell_{2,1}$~\cite{ying2009sparse}   &0.001  & 0.01 &0.001&  0.001 & 0.1  & 0.1& 0.1 \\
MDML-Tr~\cite{liu2015low}   & 0.01 & 0.01 & 0.01& 0.001  & 0.1  & 0.01& 0.1 \\
MDML-IT~\cite{davis2007information}   &  0.001& 0.1 &0.1 & 0.01  &  0.1 &0.001 &0.01  \\
MDML-Drop~\cite{qian2014distance}   &  0.01& 0.01 &0.1 & 0.001  & 0.1  &0.1 & 0.01 \\
\hline
PDML-DC~\cite{malkin2008ratio}    & 0.01 & 0.1 &0.01 & 0.01  & 0.1  &0.1 & 0.01 \\
PDML-CS~\cite{yu2011diversity}    & 0.01 & 0.1 &0.01 & 1  & 0.001  &0.001 & 0.1 \\
PDML-DPP~\cite{zou2012priors}   & 0.1 & 0.01 & 0.001&0.1   & 0.1  &0.01 & 0.1 \\
PDML-IC~\cite{bao2013incoherent}    & 0.01 & 0.001 &0.01&  0.1 &  0.01 &0.1 &0.01  \\
PDML-DeC~\cite{cogswell2015reducing}   &  0.1& 0.001 & 0.01& 0.1  & 0.01  &0.1 & 1 \\
PDML-VGF~\cite{jalali2015variational}   & 0.01 & 0.01 & 0.1&  0.1 & 0.1  &0.001 & 0.01 \\
PDML-MA~\cite{xie2015learning}   & 0.001 & 1 &0.01 & 0.01 & 0.1  & 0.01& 0.01 \\
PDML-SFN~\cite{wang2012semi,chen2017diversity}    &  0.01& 0.01 &0.01 &  0.1 & 0.1  &0.01 &  0.1\\
PDML-VND~\cite{xienear}    & 0.01 & 0.1 & 0.01& 0.001  & 0.001  &0.1 & 0.01 \\
PDML-LDD~\cite{xienear}    &0.001  & 0.01 &0.1 &  0.001 &  0.01 &0.01 & 0.01 \\
\hline
\hline
MDML-CSFN    &0.01  & 0.001 & 0.01& 0.001  & 0.1 & 0.01& 0.1 \\
MDML-CVND    & 0.01 & 0.01 &0.1 & 0.1  & 0.1  &0.01 & 0.01 \\
MDML-CLDD    & 0.01 & 0.1 &0.01 &  0.001 &  0.01 &0.001 & 0.1 \\
\hline
\end{tabular}
\caption{ Best tuned regularization parameters via cross validation.}
\label{tb:hyperpara}
\end{table}

\begin{table}
\small
\centering
\begin{tabular}{lcccccccc}
\hline
&MIMIC & EICU & Reuters &News &Cars &Birds  &Act \\
\hline
LMNN  & 3.8 &4.0&0.4 &0.7 &0.6   & 0.7  &  0.3 \\
ITML  & 12.6 &11.4& 1.2&3.2 &3.0   & 2.7  & 0.8  \\
LDML  & 3.7 &3.4&0.3 & 0.6& 0.5 &0.6 &  0.2    \\
MLEC  & 0.4 &0.4& 0.026&0.049 & 0.043  &  0.044 & 0.018  \\
GMML  & 0.5 &0.4& 0.035&0.056 &0.052   &0.049   & 0.022  \\
\hline
MDML-$\ell_2$  & 3.4 &3.5& 0.3& 0.6& 0.5  &0.6  & 0.2  \\
MDML-$\ell_1$  & 3.4 &3.6& 0.5& 0.6& 0.5  &0.6   & 0.2  \\
MDML-$\ell_{2,1}$  & 3.5 &3.7& 0.3& 0.5&  0.5 &0.6   & 0.1  \\
MDML-Tr  & 3.4 &3.7& 0.3& 0.6&0.6   &0.4   & 0.3  \\
MDML-IT  & 5.2 &5.5& 0.5&0.9 &0.8   & 1.0  & 0.4  \\
MDML-Drop  & 9.5 & 10.4& 1.2&1.7 & 1.9  & 1.7  & 0.6  \\
\hline
\end{tabular}
\caption{ Training time (hours) of additional baselines.}
\label{tb:rt}
\end{table}

\begin{table}
\small
\centering
\begin{tabular}{l|c|c|c}
\hline
&MIMIC  &EICU  &Reuters  \\
\cline{2-4}
\hline
PDML&0.654$\pm$ 0.015&0.690$\pm$ 0.009&0.963$\pm$ 0.012\\
MDML&0.659$\pm$ 0.014&0.691$\pm$ 0.005&0.962$\pm$ 0.008\\
\hline
EUC&0.558$\pm$ 0.007&0.584$\pm$ 0.008&0.887$\pm$ 0.009\\
LMNN~\cite{davis2007information} &0.643$\pm$ 0.011&0.678$\pm$ 0.007&0.951$\pm$ 0.020\\
LDML~\cite{guillaumin2009you} &0.638$\pm$ 0.017&0.678$\pm$ 0.020&0.946$\pm$ 0.009\\
MLEC~\cite{kostinger2012large} &0.633$\pm$ 0.018&0.692$\pm$ 0.008&0.936$\pm$ 0.007\\
GMML~\cite{zadehgeometric} &0.621$\pm$ 0.017&0.679$\pm$ 0.006&0.938$\pm$ 0.011\\
ILHD~\cite{carreira2016ensemble} &0.590$\pm$ 0.006&0.652$\pm$ 0.018&0.919$\pm$ 0.014\\
\hline
MDML-$\ell_2$ &0.664$\pm$ 0.019&0.706$\pm$ 0.006&0.966$\pm$ 0.012\\
MDML-$\ell_1$~\cite{qi2009efficient} &0.664$\pm$ 0.017&0.715$\pm$ 0.015&0.967$\pm$ 0.005\\
MDML-$\ell_{2,1}$~\cite{ying2009sparse} &0.658$\pm$ 0.008&0.727$\pm$ 0.016&0.970$\pm$ 0.008\\
MDML-Tr~\cite{liu2015low} &0.672$\pm$ 0.011&0.709$\pm$ 0.004&0.969$\pm$ 0.015\\
MDML-IT~\cite{davis2007information} &0.673$\pm$ 0.009&0.705$\pm$ 0.007&0.964$\pm$ 0.007\\
MDML-Drop~\cite{qian2014distance} &0.660$\pm$ 0.016&0.718$\pm$ 0.006&0.968$\pm$ 0.010\\
MDML-OS &0.665$\pm$ 0.009 & 0.711$\pm$ 0.007 & 0.968$\pm$ 0.012\\
\hline
PDML-DC~\cite{malkin2008ratio}&0.662$\pm$ 0.005&0.717$\pm$ 0.012&0.976$\pm$ 0.007\\
PDML-CS~\cite{yu2011diversity}&0.676$\pm$ 0.019&0.736$\pm$ 0.007&0.973$\pm$ 0.011\\
PDML-DPP~\cite{zou2012priors} &\tb{0.679}$\pm$ 0.008&0.725$\pm$ 0.010&0.972$\pm$ 0.015\\
PDML-IC~\cite{bao2013incoherent}&0.674$\pm$ 0.010&0.726$\pm$ 0.005&0.984$\pm$ 0.019\\
PDML-DeC~\cite{cogswell2015reducing} &0.666$\pm$ 0.007&0.711$\pm$ 0.015&0.977$\pm$ 0.011\\
PDML-VGF~\cite{jalali2015variational} &0.674$\pm$ 0.007&0.730$\pm$ 0.011&0.988$\pm$ 0.008\\
PDML-MA~\cite{xie2015learning} &0.670$\pm$ 0.009&0.731$\pm$ 0.006&0.983$\pm$ 0.007\\
PDML-SFN~\cite{wang2012semi,fu2014nokmeans,ge2014graph,chen2017diversity}&0.677$\pm$ 0.011&0.736$\pm$ 0.013&0.984$\pm$ 0.009\\
PDML-OC~\cite{liu2008output,wang2015deep}&0.663$\pm$ 0.005&0.716$\pm$ 0.010
&0.966$\pm$ 0.017\\
PDML-OS &0.658$\pm$ 0.006&0.691$\pm$ 0.004& 0.965$\pm$ 0.009\\
PDML-VND~\cite{xienear}&0.676$\pm$ 0.013&0.748$\pm$ 0.020&0.983$\pm$ 0.007\\
PDML-LDD~\cite{xienear}&0.674$\pm$ 0.012&0.743$\pm$ 0.006&0.981$\pm$ 0.009\\
\hline
MDML-CSFN&\tb{0.679}$\pm$ 0.009&0.741$\pm$ 0.011&0.991$\pm$ 0.010\\
MDML-CVND&0.678$\pm$ 0.007&0.744$\pm$ 0.005&\tb{0.994}$\pm$ 0.008\\
MDML-CLDD&0.678$\pm$ 0.012&\tb{0.750}$\pm$ 0.006&0.991$\pm$ 0.006\\
\hline
\end{tabular}
\caption{ Mean AUC and standard errors on frequent classes.}
\label{tb:auc_f}
\end{table}

\paragraph{Retrieval settings}  For each test example, we use it to query the rest of test examples based on the learned distance metric. If the distance between $\mb{x}$ and $\mb{y}$ is smaller than a threshold $s$ and they have the same class label, then this is a true positive. By choosing different values of $s$, we obtain a receiver operating characteristic (ROC) curve. For AUC on infrequent classes, we use examples belonging to infrequent classes to query the entire test set (excluding the query). AUC on frequent classes is measured in a similar way. \begin{table}
\scriptsize
\centering
\captionsetup{font=footnotesize}
\begin{tabular}{l||ccc|ccc|ccc||c|c|c|c}
\hline
&\multicolumn{3}{c|}{MIMIC}  &\multicolumn{3}{c|}{EICU}  &\multicolumn{3}{c||}{Reuters}&News&Cars&Birds&Act  \\
\cline{2-14}
&A-All & A-IF & BS &A-All & A-IF & BS &A-All & A-IF & BS &A-All&A-All&A-All& A-All\\
\hline
PDML  	&0.008&	0.019&	0.014&	0.007&	0.009&	0.010&	0.005&	0.022&	0.017&	0.005&	0.021&	0.006&	0.016\\
MDML  	&0.020&	0.006&	0.024&	0.009&	0.016&	0.009&	0.011&	0.015&	0.012&	0.008&	0.017&	0.013&	0.021\\
\hline
EUC	&0.008&	0.005&	0.012&	0.010&	0.006&	0.015&	0.017&	0.006&	0.008&	0.024&	0.016&	0.021&	0.010\\
LMNN~\cite{davis2007information} 	&0.013&	0.022&	0.009&	0.011&	0.016&	0.009&	0.014&	0.018&	0.022&	0.020&	0.011&	0.017&	0.008\\
LDML~\cite{guillaumin2009you} 	&0.025&	0.014&	0.023&	0.008&	0.005&	0.012&	0.024&	0.007&	0.011&	0.010&	0.008&	0.011&	0.005\\
MLEC~\cite{kostinger2012large} 	&0.012&	0.018&	0.016&	0.011&	0.017&	0.020&	0.005&	0.021&	0.007&	0.019&	0.007&	0.023&	0.013\\
GMML~\cite{zadehgeometric} 	&0.008&	0.011&	0.020&	0.021&	0.024&	0.013&	0.016&	0.011&	0.009&	0.008&	0.013&	0.007&	0.010\\
ILHD~\cite{carreira2016ensemble} 	&0.013&	0.017&	0.007&	0.010&	0.022&	0.004&	0.013&	0.020&	0.006&	0.018&	0.011&	0.015&	0.012\\
\hline
MDML-$\ell_2$ 	&0.016&	0.011&	0.013&	0.021&	0.005&	0.013&	0.007&	0.023&	0.016&	0.022&	0.007&	0.021&	0.025\\
MDML-$\ell_1$~\cite{qi2009efficient} 	&0.018&	0.020&	0.006&	0.013&	0.018&	0.014&	0.023&	0.006&	0.013&	0.017&	0.022&	0.018&	0.009\\
MDML-$\ell_{2,1}$~\cite{ying2009sparse} 	&0.012&	0.008&	0.017&	0.016&	0.012&	0.022&	0.015&	0.014&	0.020&	0.012&	0.024&	0.019&	0.015\\
MDML-Tr~\cite{liu2015low} 	&0.011&	0.024&	0.009&	0.022&	0.007&	0.011&	0.012&	0.007&	0.015&	0.013&	0.009&	0.018&	0.010\\
MDML-IT~\cite{davis2007information} 	&0.013&	0.009&	0.017&	0.020&	0.016&	0.021&	0.015&	0.017&	0.013&	0.019&	0.011&	0.008&	0.016\\
MDML-Drop~\cite{qian2014distance} 	&0.005&	0.014&	0.008&	0.027&	0.013&	0.016&	0.005&	0.023&	0.009&	0.008&	0.006&	0.024&	0.025\\
\hline
PDML-DC~\cite{malkin2008ratio}  	&0.008&	0.017&	0.019&	0.006&	0.015&	0.009&	0.011&	0.012&	0.018&	0.014&	0.017&	0.023&	0.008\\
PDML-CS~\cite{yu2011diversity}  	&0.019&	0.022&	0.017&	0.021&	0.023&	0.010&	0.007&	0.020&	0.016&	0.012&	0.013&	0.014&	0.022\\
PDML-DPP~\cite{zou2012priors} 	&0.014&	0.006&	0.011&	0.009&	0.008&	0.017&	0.018&	0.007&	0.013&	0.011&	0.006&	0.022&	0.005\\
PDML-IC~\cite{bao2013incoherent}  	&0.007&	0.009&	0.011&	0.006&	0.014&	0.015&	0.006&	0.017&	0.023&	0.007&	0.005&	0.019&	0.008\\
PDML-DeC~\cite{cogswell2015reducing} 	&0.019&	0.024&	0.021&	0.008&	0.006&	0.009&	0.015&	0.018&	0.006&	0.014&	0.008&	0.012&	0.018\\
PDML-VGF~\cite{jalali2015variational} 	&0.009&	0.008&	0.017&	0.013&	0.019&	0.010&	0.015&	0.009&	0.014&	0.008&	0.022&	0.021&	0.008\\
PDML-MA~\cite{xie2015learning} 	&0.021&	0.014&	0.009&	0.005&	0.019&	0.021&	0.011&	0.014&	0.016&	0.013&	0.011&	0.007&	0.009\\
PDML-SFN~\cite{wang2012semi,fu2014nokmeans,ge2014graph,chen2017diversity}  	&0.015&	0.021&	0.006&	0.022&	0.007&	0.017&	0.013&	0.010&	0.008&	0.023&	0.016&	0.024&	0.012\\
PDML-OC~\cite{liu2008output,wang2015deep}  	&0.016&	0.010&	0.011&	0.007&	0.018&	0.008&	0.019&	0.023&	0.016&	0.015&	0.011&	0.005&	0.009\\
PDML-VND~\cite{xienear}  	&0.009&	0.018&	0.007&	0.024&	0.011&	0.019&	0.021&	0.017&	0.022&	0.014&	0.006&	0.012&	0.025\\
PDML-LDD~\cite{xienear}  	&0.021&	0.012&	0.008&	0.018&	0.017&	0.013&	0.011&	0.007&	0.009&	0.007&	0.012&	0.006&	0.016\\
\hline
MDML-CSFN  	&0.011&	0.009&	0.013&	0.007&	0.008&	0.014&	0.009&	0.012&	0.008&	0.025&	0.007&	0.004&	0.011\\
MDML-CVND  	&0.006&	0.007&	0.011&	0.012&	0.014&	0.009&	0.012&	0.013&	0.006&	0.009&	0.011&	0.014&	0.013\\
MDML-CLDD  	&0.009&	0.012&	0.011&	0.010&	0.005&	0.013&	0.018&	0.005&	0.012&	0.011&	0.015&	0.008&	0.010\\
\hline
\end{tabular}
\caption{ Standard errors.}
\label{tb:se}
\end{table}

\begin{table}
\scriptsize
\centering
\captionsetup{font=footnotesize}
\begin{tabular}{l||c|c|c||c|c|c|c}
\hline
&MIMIC &EICU  &Reuters&News&Cars&Birds&Act  \\
\hline
PDML    &  0.175  &  0.145  &  0.043  &  0.095  &  0.149  &  0.075  &  0.045\\
MDML    &  0.187  &  0.142  &  0.045  &  0.087  &  0.124  &  0.066  &  0.042\\
\hline
LMNN   &  0.183  &  0.153  &  0.031  &  0.093  &  0.153  &  0.073  &  0.013\\
LDML   &  0.159  &  0.139  &  0.034  &  0.079  &  0.131  &  0.072  &  0.068\\
MLEC   &  0.162  &  0.131  &  0.042  &  0.088  &  0.151  &  0.039  &  0.043\\
GMML   &  0.197  &  0.157  &  0.051  &  0.063  &  0.118  &  0.067  &  0.036\\
ILHD   &  0.164  &  0.162  &  0.048  &  0.077  &  0.117  &  0.045  &  0.059\\
\hline
MDML-$\ell_2$   &  0.184  &  0.136  &  0.037  &  0.072  &  0.105  &  0.053  &  0.041\\
MDML-$\ell_1$   &  0.173  &  0.131  &  0.042  &  0.064  &  0.113  &  0.061  &  0.026\\
MDML-$\ell_{2,1}$   &  0.181  &  0.129  &  0.034  &  0.073  &  0.121  &  0.044  &  0.024\\
MDML-Tr   &  0.166  &  0.138  &  0.024  &  0.076  &  0.111  &  0.058  &  0.037\\
MDML-IT   &  0.174  &  0.134  &  0.033  &  0.061  &  0.109  &  0.036  &  0.013\\
MDML-Drop   &  0.182  &  0.140  &  0.021  &  0.076  &  0.114  &  0.063  &  0.024\\
MDML-OS   &  0.166  &  0.133  &  0.032  &  0.063  &  0.108  &  0.057  &  0.031\\
\hline
PDML-DC    &  0.159  &  0.131  &  0.035  &  0.069  &  0.127  &  0.064  &  0.035\\
PDML-CS    &  0.163  &  0.135  &  0.031  &  0.083  &  0.103  &  0.045  &  0.033\\
PDML-DPP   &  0.147  &  0.140  &  0.038  &  0.067  &  0.117  &  0.072  &  0.041\\
PDML-IC    &  0.155  &  0.127  &  0.018  &  0.075  &  0.116  &  0.074  &  0.029\\
PDML-DeC   &  0.164  &  0.123  &  0.023  &  0.082  &  0.125  &  0.051  &  0.033\\
PDML-VGF   &  0.158  &  0.136  &  0.014  &  0.064  &  0.136  &  0.035  &  0.028\\
PDML-MA   &  0.143  &  0.128  &  0.023  &  0.078  &  0.102  &  0.031  &  0.042\\
PDML-OC    &  0.161  &  0.142  &  0.032  &  0.061  &  0.111  &  0.063  &  0.034\\
PDML-OS   &  0.169  &  0.137  &  0.015  &  0.083  &  0.119  &  0.058  &  0.042\\
PDML-SFN    &  0.153  &  0.126  &  0.022  &  0.069  &  0.127  &  0.043  &  0.028\\
PDML-VND    &  0.148  &  0.135  &  0.019  &  0.078  &  0.116  &  0.067  &  0.035\\
PDML-LDD    &  0.146  &  0.121  &  0.017  &  0.054  &  0.111  &  0.036  &  0.021\\
\hline
MDML-CSFN    &  0.142  &  0.124  &  0.019  &  0.062  &  0.092  &  0.043  &  0.019\\
MDML-CVND    &  0.137  &  0.115  &  0.008  &  0.055  &  0.094  &  0.038  &  0.013\\
MDML-CLDD    &  0.131  &  0.118  &  0.012  &  0.058  &  0.089  &  0.026  &  0.016\\
\hline
\end{tabular}
\caption{ The gap of training AUC and testing AUC (training-AUC minus testing-AUC)}
\label{tb:gap}
\end{table}

For computational efficiency, in MDML-based methods, we do not use $(\mb{x}-\mb{y})^\top\mb{M}(\mb{x}-\mb{y})$ to compute distance directly. Given the learned matrix $\mb{M}$ (which is of rank $k$), we can decompose it into $\mb{L}^\top\mb{L}$ where $\mb{L}\in\mathbb{R}^{k\times d}$. Let $\mb{U}\bs\Lambda\mb{U}^\top$ be the eigen-decomposition of $\mb{M}$. Let $\lambda_1,\cdots,\lambda_k$ denote the $k$ nonzero eigenvalues and $\mb{u}_i,\cdots,\mb{u}_k$ denote the corresponding eigenvectors. Then $\mb{L}$ is the transpose of $[\sqrt{\sigma_1}\mb{u}_1,\cdots,\sqrt{\sigma_k}\mb{u}_k]$. Given $\mb{L}$, we can use it to transform each input $d$-dimensional feature vector $\mb{x}$ into a new $k$-dimensional vector $\mb{Lx}$, then perform retrieval on the new vectors based on Euclidean distance. Note that only when computing Euclidean distance between $\mb{Lx}$ and $\mb{Ly}$, we have that $\|\mb{Lx}-\mb{Ly}\|_2^2$ is equivalent to $(\mb{x}-\mb{y})^\top\mb{M}(\mb{x}-\mb{y})$. For other distances or similarity measures between $\mb{Lx}$ and $\mb{Ly}$, such as L1 distance and cosine similarity, this does not hold. Performing retrieval based on $\|\mb{Lx}-\mb{Ly}\|_2^2$ is more efficient than that based on $(\mb{x}-\mb{y})^\top\mb{M}(\mb{x}-\mb{y})$ when $k$ is smaller than $d$. Given $m$ test examples, the computation complexity of $\|\mb{Lx}-\mb{Ly}\|_2^2$ based retrieval is $O(mkd+m^2k)$, while that of $(\mb{x}-\mb{y})^\top\mb{M}(\mb{x}-\mb{y})$ based retrieval is $O(m^2d^2)$. 

\paragraph{Additional details of baselines}
In the Large Margin Nearest Neighbor (LMNN) DML method~\cite{weinberger2005distance}, there is a nonconvex formulation and a convex formulation. We used the convex one.  

Though the variational Gram function (VGF)~\cite{jalali2015variational} is convex, when it is used to regularize PDML, the overall problem is non-convex and it is unclear how to seek a convex relaxation. In Geometric Mean Metric Learning (GMML)~\cite{zadehgeometric}, the prior matrix was set to an identity matrix. In Independent Laplacian Hashing with Diversity (ILHD)~\cite{carreira2016ensemble}, we use the ILTitf variant. The hash function is kernel SVM with a radial basis function (RB) kernel. We did not compare with unsupervised hashing methods~\cite{weiss2009spectral,kong2012isotropic,gong2013iterative,ge2014optimized,ji2014batch}.

\paragraph{Hyperparameters}
Table \ref{tb:hyperpara} shows the best tuned regularization parameters on different datasets. In Orthogonal Constraints (OR)~\cite{liu2008output,wang2015deep}, there is no regularization parameter. In dropout~\cite{qian2014distance}, the regularization parameter designates the probability of dropping elements in the Mahalanobis matrix. In LMNN, the weighting parameter $\mu$ was set to 0.5. In GMML~\cite{zadehgeometric}, the regularization parameter $\lambda$ was set to 0.1. The step length $t$ of geodesic was set to 0.3. In ILHD~\cite{carreira2016ensemble}, the scale parameter of the RBF kernel was set to 0.1.

\subsection{Additional Experimental Results}
\label{sec:res}

\paragraph{Training time of other baselines} Table \ref{tb:rt} shows the training time of additional baselines.

\paragraph{AUC on frequent classes} Table \ref{tb:auc_f}
shows the mean AUC and standard errors on frequent classes. MDML-(CSFN,CVND,CLDD) achieve better mean AUC than the baselines.

\paragraph{Standard errors} 
Table~\ref{tb:se} shows the standard errors of AUC on all classes and infrequent classes and standard errors of balance scores.

\paragraph{Gap between training AUC and testing AUC}
Table~\ref{tb:gap} shows the gap between training AUC and testing AUC (training-AUC minus testing-AUC).

\paragraph{Convex regularizers versus the nonconvex regularizers} We verify that the convex regularizers including CSFN, CLDD and CVND are good approximations of the nonconvex regularizers including SFN, LDD and VND. Figure \ref{fig:approx} shows that these nonconvex regularizers consistently decrease as we increase the regularization parameter in MDML-(CSFN,CLDD,CVND). Since MDML-(CSFN,CLDD,CVND) are to be minimized, increasing the regularization parameter decreases CSFN,CLDD,CVND. In other words, these figures show that smaller CSFN,CLDD,CVND leads to smaller SFN, LDD and VND. This demonstrates that CSFN,CLDD,CVND are good approximations of SFN,LDD and VND.

\begin{figure*}[t]
\begin{center}
\includegraphics[width=0.3\textwidth]{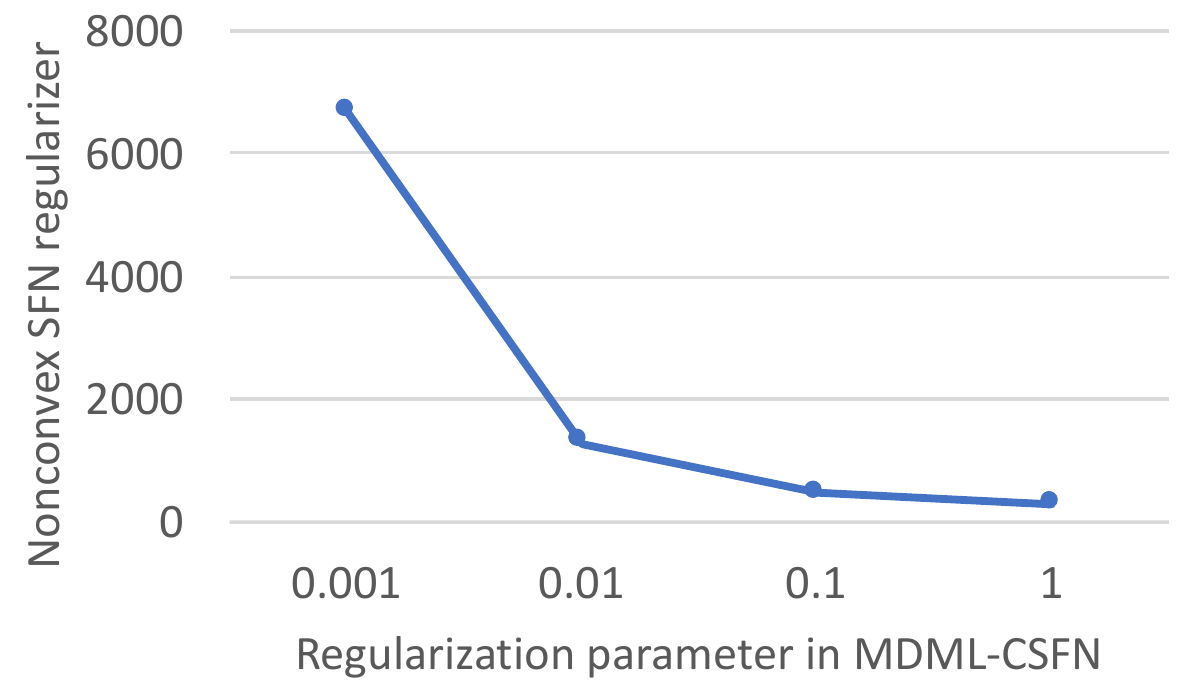}
\hspace{0.1in}
\includegraphics[width=0.33\textwidth]{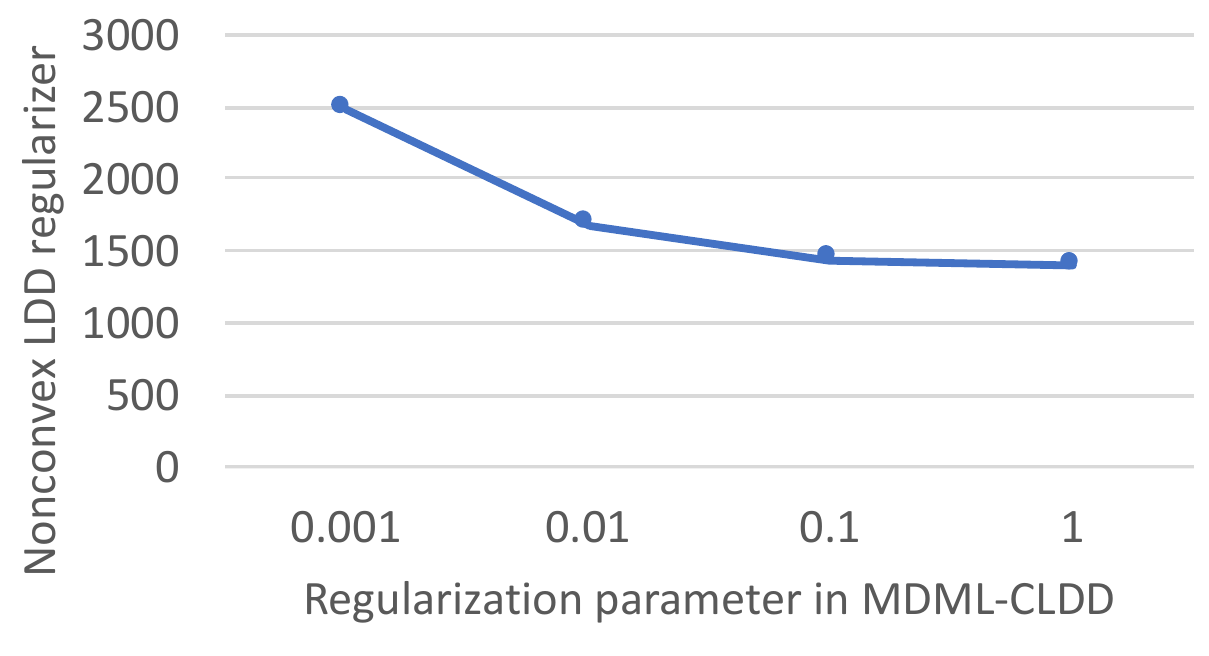}
\hspace{0.1in}
\includegraphics[width=0.3\textwidth]{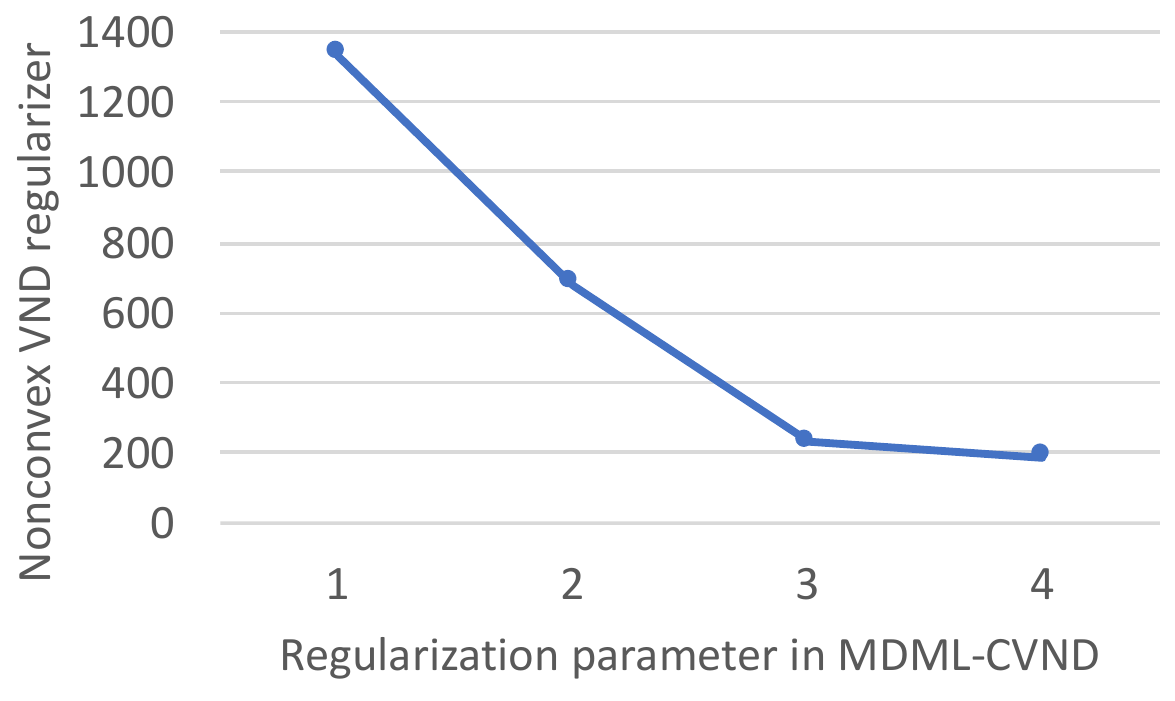}
\end{center}
\caption{(Left) Nonconvex SFN regularizer versus the regularization parameter in MDML-CSFN; (Middle) Nonconvex LDD regularizer versus the regularization parameter in MDML-CLDD; (Right) Nonconvex VND regularizer versus the regularization parameter in MDML-CVND.}
\label{fig:approx}
\end{figure*}

\paragraph{Additional experimental analysis
}
\begin{itemize}
    \item \tb{Training time} Unregularized PDML runs faster that regularized PDML methods because it has no need to tune the regularization parameter, which reduces the number of experimental runs by 4 times. Unregularized MDML runs faster than regularized MDML methods because it has no need to tune the regularization parameter or the number of projection vectors, which reduces the number of experimental runs by 12 times. PDML-(DC,DPP,VND,LDD) takes longer time than other regularized PDML methods since they need eigendecomposition to compute the gradients. PDML-OC has no regularization parameter to tune, hence its number of experimental runs is 4 times fewer than other regularized PDML methods.
    \item \tb{Balancedness} In most DML methods, the AUC on infrequent classes is worse than that on frequent classes, showing that DML is sensitive to the imbalance of pattern-frequency, tends to be biased towards frequent patterns and is less capable to capture infrequent patterns. This is in accordance with previous study~\cite{xie2015diversifying}.
\end{itemize}

{\small
\bibliographystyle{plain}
\bibliography{refs}
}

\end{document}